\documentclass{article}

\usepackage[preprint]{neurips_2026}

\usepackage[T1]{fontenc}
\usepackage{times}
\usepackage{amsmath,amssymb,amsthm,mathtools}
\usepackage{booktabs,multirow}
\usepackage{graphicx}
\usepackage{xcolor}
\usepackage{enumitem}
\usepackage{tikz}
\usepackage{url}
\usepackage{float}
\usepackage{array}
\usepackage{hyperref}

\definecolor{wsmblue}{HTML}{4477AA}
\definecolor{qadorange}{HTML}{EE7733}
\definecolor{decaygreen}{HTML}{228833}
\definecolor{todoamber}{HTML}{996515}

\newcommand{\QAM}{\textsc{QAM}}
\newcommand{\WSM}{\textsc{WSM}}

\newcommand{\E}{\mathbb{E}}
\newcommand{\Cov}{\operatorname{Cov}}

\theoremstyle{plain}
\newtheorem{theorem}{Theorem}

\theoremstyle{definition}

\newtheorem{assumption}[theorem]{Assumption}
\theoremstyle{remark}
\newtheorem{remark}[theorem]{Remark}
\newcolumntype{L}[1]{>{\raggedright\arraybackslash}p{#1}}

\title{QAM: Quadratic-Accurate Checkpoint Merging\\via Sequential Consistency}

\author{
    \textbf{Shihao Wang$^{1,2}$, Rui Kong$^{2}$, Xinran Chen$^{2}$, Hui Wu$^{2}$, Qipeng Qian$^{3}$, Jinman Zhao$^{4}$,}\\
    \textbf{Jiashu  Zhao$^{5}$, Yuchen Li$^{2}$, Jimmy Huang$^{6}$, Dawei Yin$^{2}$} \\
    $^1$University of Wisconsin -- Madison \quad $^2$Baidu Inc. \quad $^3$University of Arizona\\
    $^4$University of Toronto \quad $^5$Wilfrid Laurier University \quad $^6$York University\\ 
    \texttt{yuchenli1230@gmail.com,yindawei@acm.org}
}

\begin{document}
\maketitle

\begin{abstract}
Saved checkpoints record states along a training trajectory, but generally
do not determine the updates at states that would be visited under a
different schedule. We study how accurately these
checkpoints can reconstruct the endpoint of a sequential reference with
prescribed update strengths.
Under a common local transition model, two checkpoint-index moment
conditions characterize all convex merges that agree with this reference
through second order. We then prove an information limit that for
nondegenerate profiles, no algorithm using only a fixed-length
gradient-descent (GD) history with step size $h$ can achieve $o(h^3)$
endpoint error uniformly over a fixed class of smooth, strongly convex
losses. The lower bound follows from two losses with identical GD
checkpoint histories but sequential reference endpoints separated by
$\Omega(h^3)$.
\textbf{Quadratic-Accurate Merging} (QAM) achieves a matching
uniform $O(h^3)$ endpoint error bound.
Its explicit coefficients also define the unique profile-dependent merge that exactly matches the sequential GD reference across
all fixed quadratic objectives.
Across two public Adam checkpoint trajectories (SmolLM3-3B and OpenEuroLLM-Prelude-9B), three windows
and three profiles per model, and 15 tasks, QAM shows mixed results
for short windows and broader advantages over
\textbf{Warmup-Stable and Merge} (WSM) for longer windows.
Matched-moment GSM8K diagnostics further show that local consistency
alone does not fully determine downstream scores.
These results characterize the reconstruction limits of saved histories,
provide a coefficient rule that attains the optimal rate, and assess
its practical utility.
\end{abstract}

\section{Introduction}
\label{sec:introduction}
Large language model pretraining often requires adjustments to the training plan, such as introducing new data, revising the data mixture, or extending the training budget. A learning-rate decay schedule fixed in advance makes it harder to accommodate these adjustments. In contrast, warmup-stable schedules offer greater flexibility by maintaining a constant learning rate after warmup.
Averaging checkpoints saved during the stable phase can improve
model quality without further optimization
\citep{izmailov2018averaging,sanyal2023early,ajroldi2025average}.
However, choosing the averaging coefficients raises a more
fundamental question:
what alternative training behavior can a single saved history reconstruct?
The history reveals updates at the recorded states,
but not the updates at states that a changed schedule would have visited.

Warmup-Stable and Merge (\WSM{})~\citep{tian2025wsm} gives an exact
tail-sum correspondence between checkpoint weights and strengths assigned
to recorded updates. Holding this profile fixed, we define a sequential
reference that recomputes each relaxed transition at its current state.
At checkpoint scale, a transition represents a complete saved interval.
This reference supplies a coefficient-design criterion while applying its
weights to a sparse Adam history remains a heuristic.

We first characterize the entire class of convex merges that agrees with
this reference through second order under a common local transition
model. We show that two index moments are necessary and sufficient. The base merge's
discrepancy coefficient is a sum of covariances between nested update
indicators. Cancelling it preserves the index mean and reduces the
variance by twice this sum. The resulting family generally contains
many coefficient vectors.

Our main limitation result concerns the information in the training trajectory itself. For any nondegenerate fixed profile, we construct two globally
smooth, strongly convex losses with exactly the same GD checkpoints,
checkpoint gradients, and checkpoint losses, but reference endpoints
separated by $\Omega(h^3)$. Hence no reconstruction algorithm, including
a nonlinear or state-dependent one, has uniformly smaller error on the
specified smooth class as $h\to0$. A second-order-consistent merge attains
the matching $O(h^3)$ upper bound.

Quadratic-Accurate Merging (\QAM{}) realizes this rate using the
distribution of a sum of independent Bernoulli variables with the
prescribed update strengths. Local quadratic models of pretraining loss
motivate an additional selection principle~\citep{li2025modelmerging} that exact matching across every quadratic objective uniquely fixes the
complete coefficients. Our characterization and information bound locate
the attainable local accuracy of this rule among all reconstructions from
the same observations. Moreover, QAM itself is simple to construct as it requires only a scalar convolution on the coefficients.

We evaluate its practical utility on two public 3B/9B Adam pretraining
trajectories, three windows and three profiles per history, and 15 tasks.
Compared with paired WSM, results are mixed at short windows and more
favorable at intermediate and long windows. The four intermediate- and long-window linear-profile QAM merges exceed the latest checkpoint in 45/60 task comparisons. A complementary
GSM8K diagnostic changes full coefficients while holding their first two
index moments fixed. Its score differences demonstrate the empirical
freedom left by second-order consistency. Our experiments evaluate QAM's practical performance, while our
theory establishes its optimal reconstruction rate under the
stated assumptions.
\section{Related Work}
\label{sec:related-work}

\paragraph{Checkpoint and weight averaging.}
Stochastic weight averaging combines models along a training
trajectory \citep{izmailov2018averaging}, while later work studies
denser or more diverse forms of weight averaging
\citep{cha2021swad,rame2022diwa}.
LAWA studies early averaging with high learning rates in LLM
pretraining \citep{sanyal2023early}, and
\citet{ajroldi2025average} examine when and where averaging helps.
Beyond same-trajectory averaging, model merging has been studied
through model soups, Fisher-weighted merging, task arithmetic,
interference-aware merging, and sparsified parameter updates
\citep{wortsman2022soups,matena2022fisher,ilharco2023task,
yadav2023ties,yu2024dare}; see \citet{yang2024merging} for a broader
survey.

\paragraph{Averaging and update schedules.}
\citet{sandler2023trajectories} connect iterate averaging with learning-rate
schedules through models of SGD trajectories, while
\citet{meterez2026anytime} study horizon-free schedules together with weight
averaging.
For stable-phase LLM pretraining, \citet{tian2025wsm} derive an exact
tail-sum correspondence between checkpoint weights and recorded-update
weights.
We start from this correspondence and account for the state dependence
introduced by sequentially recomputed transitions.

\paragraph{Quadratic models and training dynamics.}
\citet{li2025modelmerging} explain pretraining averaging through a
second-order loss expansion, Hessian geometry, and interactions between
checkpoint deviations. Quadratic models also clarify averaging and
learning-rate schedules~\citep{sandler2023trajectories}. This viewpoint
motivates our matching criterion. Notice that the resulting coefficients do not
require estimating the Hessian or fitting a transition to the history.

\section{Quadratic-Accurate Checkpoint Merging}
\label{sec:induced-dynamics}
\label{sec:QAM}

We hold an update profile fixed and ask for a merge consistent with
applying those update strengths sequentially. This distinguishes the
choice of a design criterion from the downstream evaluation.

\subsection{A sequential reference at checkpoint scale}
\label{sec:qam-reference}

Let $\theta_0,\ldots,\theta_k$ be chronologically ordered checkpoints,
$k\geq1$, and $c_i\geq0$ with $\sum_i c_i=1$ a convex base merge.
Writing $d_j=\theta_{j+1}-\theta_j$ yields WSM's exact correspondence
between checkpoint weights and recorded-update weights~\citep{tian2025wsm}:
\begin{equation}
\theta_c:=\sum_{i=0}^{k}c_i\theta_i
=\theta_0+\sum_{j=0}^{k-1}W_jd_j,
\qquad W_j:=\sum_{i=j+1}^{k}c_i.
\label{eq:qam-base-profile}
\end{equation}
\begin{equation}
c_0=1-W_0,\qquad
c_i=W_{i-1}-W_i\ (1\leq i<k),\qquad c_k=W_{k-1}.
\label{eq:inverse-tail-sum}
\end{equation}
Thus $c$ fixes a nonincreasing input profile $W\in[0,1]^k$.
Reweighting the recorded $d_j$ leaves their evaluation states unchanged.
To model the effect of changing those states, take a common transition
$F$, with $\theta_{i+1}=F(\theta_i)$, and define
\begin{equation}
\phi_0^W=\theta_0,\qquad
\phi_{j+1}^W=\phi_j^W+W_j\bigl[F(\phi_j^W)-\phi_j^W\bigr].
\label{eq:checkpoint-reference}
\end{equation}
Each transition is recomputed at the state produced by earlier ones.
One transition represents a whole checkpoint interval. For a single GD
step, this is learning-rate scaling. Similarly, for a saved block, it is the relaxation
of the block displacement. Agreement with this reference is our
\emph{sequential-consistency criterion} for selecting coefficients.

\subsection{Local discrepancy and its coefficient geometry}
\label{sec:qam-discrepancy}

\begin{assumption}[Local checkpoint transition]
\label{ass:common-map}
On a neighborhood of the initial state, the common transition obeys
\begin{equation}
F_h(\theta)=\theta+h f(\theta)+h^2 a(\theta)+O(h^3)
\label{eq:common-map}
\end{equation}
uniformly as $h\to0$, with $f\in C^2$, $a\in C^1$ and bounded local
derivatives. The initial state, $k$, and $W$ remain fixed.
\end{assumption}

Here $h$ measures the local size of a complete transition, and $a$
includes its second-order within-interval contribution. Set
$f_0=f(\theta_0)$, $a_0=a(\theta_0)$, $J_0=Df(\theta_0)$, and
\begin{equation}
A_W:=\sum_{j=0}^{k-1}W_j,
\qquad
B_W:=\sum_{0\leq\ell<j\leq k-1}W_\ell W_j.
\label{eq:qam-profile-sums}
\end{equation}
Expanding both finite trajectories (Appendix~\ref{app:common-proof}),
the base merge matches the terms in $f_0$ and $a_0$, but its
interaction term differs:
\begin{equation}
\theta_c-\phi_k^W=h^2\Delta_2(W)J_0f_0+O(h^3),\qquad
\Delta_2(W):=\sum_{\ell<j}W_j(1-W_\ell)\geq0.
\label{eq:common-discrepancy}
\end{equation}
The discrepancy has a direct interpretation. Consider a checkpoint index
$I_c\sim c$ and set $X_j=\mathbf1\{I_c>j\}$. The indicators are nested, namely
selecting an update also selects all earlier ones. Thus
$\E X_j=W_j$ and $\E[X_\ell X_j]=W_j$ for $\ell<j$, whereas the
sequential reference weights this interaction by $W_\ell W_j$.
Consequently $\Delta_2(W)=\sum_{\ell<j}\Cov(X_\ell,X_j)$.
This identifies a coefficient correction without estimating $J_0f_0$.

\subsection{QAM and second-order consistency}
\label{sec:qam-construction}

Replace the nested indicators by independent
$Y_j\sim\operatorname{Bernoulli}(W_j)$ and use the distribution of
$I_q=\sum_jY_j$ as checkpoint weights. Its generating polynomial is
\begin{equation}
\boxed{
Q_W(t):=\prod_{j=0}^{k-1}(1-W_j+W_jt)
=\sum_{i=0}^{k}q_it^i,
\qquad
\theta_{\mathrm{QAM}}:=\sum_{i=0}^{k}q_i\theta_i.
}
\label{eq:qam-weights}
\end{equation}
The coefficients are nonnegative and sum to $Q_W(1)=1$.
Independence gives $\E I_q=A_W$ and $\E\binom{I_q}{2}=B_W$,
matching both interaction terms in the reference.

\begin{theorem}[Local sequential consistency]
\label{thm:common-consistency}
For fixed $k,W$ and convex weights $p$ independent of $h$ and the
transition, $\sum_i p_i\theta_i-\phi_k^W=O(h^3)$ for every transition
satisfying Assumption~\ref{ass:common-map} if and only if
\[
\E_p I=A_W,\qquad \E_p\binom I2=B_W.
\]
Equivalently, $p$ shares QAM's index mean and variance. In particular,
$\theta_{\mathrm{QAM}}-\phi_k^W=O(h^3)$.
\end{theorem}

\begin{figure}[H]
    \centering
    \includegraphics[width=0.8\linewidth]{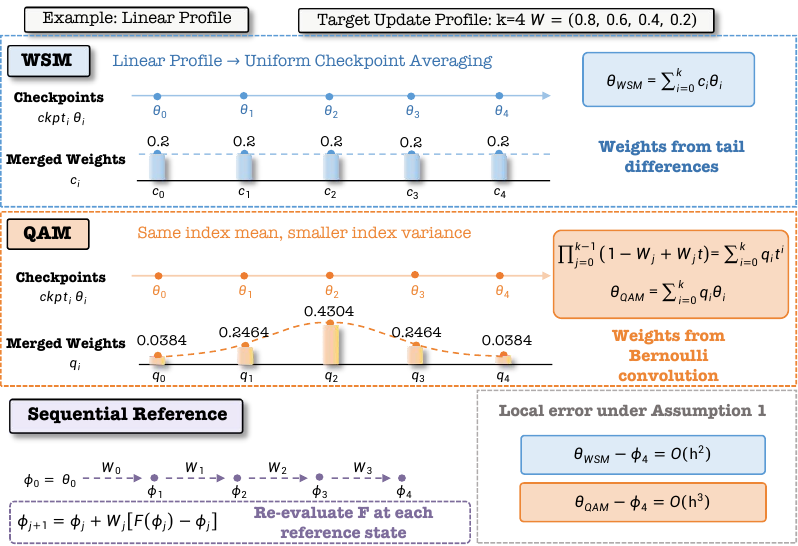}
    \caption{Illustrative example of QAM with a linearly decaying update profile $W_j$.}
\end{figure}
The characterization separates the consistency criterion from its
realization. For $k\geq3$ and $W_j\in(0,1)$, the convex solution family
has dimension $k-2$: normalization and the two moments impose three
independent linear constraints on $k+1$ coefficients, and $q_i>0$.
Quadratic matching below selects one member of this family.
QAM preserves
the base merge's mean and changes its variance by exactly
\begin{equation}
\operatorname{Var}(I_c)-\operatorname{Var}(I_q)
=2\Delta_2(W)\geq0.
\label{eq:qam-variance-gap}
\end{equation}
Thus half the variance reduction equals the weight-dependent factor
in Eq.~\eqref{eq:common-discrepancy}. This is a geometric interpretation
of the consistency correction, rather than a monotonic relationship
between variance and downstream accuracy.

\subsection{INFORMATION LIMIT FOR CHECKPOINT-ONLY RECONSTRUCTION}
\label{sec:order-barrier}

The attainable order is limited by the observations, even if the merge
can depend nonlinearly on the states. Consider ordinary GD with step
$h$, starting at $0\in\mathbb R^2$, and the reference GD steps $hW_j$.
For a fixed $M>0$, define the fixed loss class
\begin{equation}
\mathcal L_M=\left\{L\in C^4(\mathbb R^2):
\begin{array}{l}
\tfrac12 I\preceq\nabla^2L\preceq\tfrac52 I,\quad\|\nabla L(0)\|\leq3,\\
\|D^3L\|_\infty,\|D^4L\|_\infty\leq M
\end{array}\right\}.
\label{eq:loss-class}
\end{equation}
Let $T_h(L)=(\theta_0,\ldots,\theta_k)$ be the recorded history.
The best worst-case reconstruction error from this information is
\begin{equation}
R_h=\inf_{\mathcal A}\sup_{L\in\mathcal L_M}
\|\mathcal A(T_h(L),h,W)-\phi_k^W(L)\|,
\label{eq:checkpoint-risk}
\end{equation}
where $\mathcal A$ is any algorithm using only the displayed inputs.
\begin{remark}
The constants $1/2$, $5/2$, and $3$ are convenient choices
for the construction and have no special threshold significance.
\end{remark}
\begin{theorem}[Checkpoint-information bound]
\label{thm:information-bound}
Fix $k\geq2$, $W\in[0,1]^k$, and $M>0$. If
$S_W:=\sum_{\ell<j}W_jW_\ell(1-W_\ell)>0$, there are
$c,C,h_0>0$ such that
$c h^3\leq R_h\leq C h^3$ for $0<h<h_0$.
QAM attains the upper bound. The lower bound continues to hold if
checkpoint gradients and loss values are also supplied.
If $S_W=0$, QAM reconstructs the reference exactly.
\end{theorem}

\paragraph{The history is insufficient.}
Set $r(x,y)=y-x^2-2x$ and
$L_h(x,y)=\tfrac12(1+x)^2+\tfrac12(2-h)(1+y)^2$.
Its ordinary iterates
$((1-h)^i-1,(1-h)^{2i}-1)$ lie on $r=0$.
Adding $\varepsilon r^2/2$ near this curve, with a smooth cutoff,
preserves every recorded state, gradient, and loss. Both objectives
belong to the same $\mathcal L_M$ for small fixed $\varepsilon>0$.
Their reference endpoints nevertheless differ by
$\varepsilon S_Wh^3(2,-1)+O(h^4)$.
The same observations must produce the same reconstruction, so one
error is at least half this separation. The matching uniform upper
bound follows from QAM's moment conditions.
Appendix~\ref{app:information-bound} gives the complete construction.

The hard objectives may vary with $h$ inside the fixed class; the result
is a uniform minimax rate, not a pointwise lower bound for a fixed loss.
Other moment-matching rules can attain the rate, which does not order
their constants or task scores. For fixed linear weights, two third-order
terms give a corresponding algebraic obstruction; QAM's leading remainder
is proportional to $S_WD^2f[f,f]$
(Appendix~\ref{app:order-barrier}). Quadratic exactness below selects a
complete rule where this nonlinear term vanishes.

\subsection{Quadratic exactness selects a complete rule}
\label{sec:qam-quadratic-exactness}

Local quadratic loss models motivate a stronger criterion
\citep{li2025modelmerging}. Let
\begin{equation}
L(\theta_0+u)=L(\theta_0)+g_0^\top u+\tfrac12u^\top H_0u,
\qquad H_0=H_0^\top.
\label{eq:qam-quadratic-model}
\end{equation}
be fixed, and consider ordinary GD with step $\eta$ and reference GD
with steps $\eta W_j$, both initialized at $\theta_0$.
\begin{theorem}[Quadratic endpoint matching and uniqueness]
\label{thm:qam-quadratic-exactness}
Fix $k\geq1$, $W\in[0,1]^k$, and $\eta>0$.
For every fixed quadratic objective and initial state, the
GD trajectories in Eqs.~\eqref{eq:qam-local-map}
and~\eqref{eq:qam-reference} satisfy
\begin{equation}
\sum_{i=0}^{k}q_i\theta_i=\phi_k^W.
\label{eq:qam-quadratic-endpoint}
\end{equation}
If real coefficients $p_0(W),\ldots,p_k(W)$ depending only on $W$
satisfy this identity with $p$ in place of $q$ for every quadratic
objective and every initial state, then $p_i(W)=q_i(W)$ for all $i$.
\end{theorem}

For $L(x)=\tfrac12\lambda x^2$, set $t=1-\eta\lambda$. Ordinary GD yields
$x_i=t^ix_0$, whereas each reference step multiplies its state
by $1-W_j+W_jt$. Expanding their product gives the coefficients in
Eq.~\eqref{eq:qam-weights}; matching this polynomial for every scalar
quadratic proves uniqueness. Appendix~\ref{app:gd-derivation} gives the
full vector proof, including linear terms and singular or indefinite Hessians.

\paragraph{Coefficient construction.}
\label{sec:qam-implementation}
Initialize $q^{(0)}_0=1$, with zero entries outside the support, and
apply the convolution
\begin{equation}
q_i^{(j+1)}=(1-W_j)q_i^{(j)}+W_jq_{i-1}^{(j)},
\qquad i=0,\ldots,j+1.
\label{eq:qam-coefficient-recurrence}
\end{equation}
This costs $O(k^2)$ scalar operations and $O(k)$ working memory.
The final merge uses the saved checkpoints without gradients,
optimizer states, a fitted transition, or additional training.
For actual Adam histories, a common local transition is a design model;
the experiments assess the resulting weights, not reference-endpoint accuracy.

\section{Experiments}
\label{experiment}

We evaluate the practical utility of QAM on saved pretraining histories,
using matched WSM merges and individual checkpoints as references.
These comparisons assess downstream quality rather than reconstruction
of a separately trained reference endpoint.

\subsection{Setup}
\label{sec:setup}

\paragraph{Trajectories and profiles.}
We study two public models, SmolLM3-3B and
OpenEuroLLM-Prelude-9B, using checkpoints from constant-learning-rate
training (Table~\ref{tab:settings}). The SmolLM3 windows contain 5, 10, or 15 checkpoints
and span approximately 0.40, 0.90, or 1.40T training tokens.
The Prelude windows contain 10, 20, or 40 checkpoints and span
0.36, 0.76, or 1.56T tokens. We call these the short, intermediate,
and long windows within each trajectory.
For $k+1$ checkpoints, ordered oldest to newest,
we set $W_j=D((j+1)/(k+1))$ for $j=0,\ldots,k-1$ and use
\[
D_{\mathrm{linear}}(s)=1-s,\qquad
D_{\mathrm{cosine}}(s)=\cos(\pi s/2),\qquad
D_{1-\sqrt{\cdot}}(s)=1-\sqrt{s}.
\]
Paired WSM and QAM merges use the same checkpoints and profile,
changing only the coefficients in Eqs.~\eqref{eq:inverse-tail-sum}
and~\eqref{eq:qam-weights}. Appendix~\ref{app:weights} shows the resulting checkpoint weights
across window lengths and profiles.
All merges use \texttt{fp32} accumulation and \texttt{bf16} export,
without additional training or coefficient fitting. 

\paragraph{Evaluation and reporting.}
The suite covers 15 tasks in six capability groups including Knowledge,
Commonsense, Reading, Math, Semantics, and Dialogue.
MMLU and GSM8K use 5-shot evaluation, and other accuracy tasks use zero-shot.
Group scores weight accuracy tasks by evaluation-set size. Math averages
flexible and strict GSM8K extraction, and Dialogue uses MuTual MRR.
Appendix~\ref{app:experimental-details} gives task membership and
checkpoint ranges.
We report differences and win counts over the
evaluated grid. Configurations share training histories and evaluation
examples, and these counts should not be interpreted as independent replications.

\begin{table}[H]
\centering
\small
\caption{Model details.}
\label{tab:settings}
\begin{tabular}{@{}L{2.35cm}L{1.35cm}L{3.95cm}L{1.55cm}L{2.35cm}@{}}
\toprule
Model & Size & Checkpoints & Interval & Windows \\
\midrule
\textbf{SmolLM3}
& $\sim 3.075$B
& last 15 constant-LR checkpoints; iter 3.64M--4.20M
& $\sim 100$B tokens
& Last-5/10/15 \\
\addlinespace[0.25em]
\textbf{OpenEuroLLM Prelude}
& $\sim 9.095$B
& last 40 constant-LR checkpoints; iter 859200--952800
& $\sim 40$B tokens
& Last-10/20/40 \\
\bottomrule
\end{tabular}
\end{table}

\subsection{Utility across capabilities}
\label{sec:capability-utility}

Against the mean constituent score, which is the expected score
of selecting one checkpoint uniformly from the window, QAM improves in 105/108 capability
comparisons, including every Knowledge, Commonsense, Math, and
Semantics configuration (Tables~\ref{tab:stage2_group_vs_envelope}
and~\ref{tab:prelude_group_vs_envelope}). The only three regressions are in SmolLM3 Reading and Dialogue. We next
compare coefficient rules and stronger individual-checkpoint references.

\subsection{Window-dependent comparison with WSM}
\label{sec:window-results}

Figures~\ref{fig:heatmap:smollm3} and~\ref{fig:heatmap:prelude} show
all paired capability-group differences across windows and profiles.
Short-window differences are small and mixed as the median absolute gap
is about 0.35 points, with QAM lower in 21 of 36 comparisons. This is the window scale emphasized
in prior work~\citep{tian2025wsm}.
At intermediate windows, QAM wins 28 of 36 comparisons, including
22 of 30 when Math is excluded. At long windows, it wins 29 of 36,
including 23 of 30 outside Math.

\begin{figure}[htbp]
    \centering
    \includegraphics[width=0.95\textwidth]{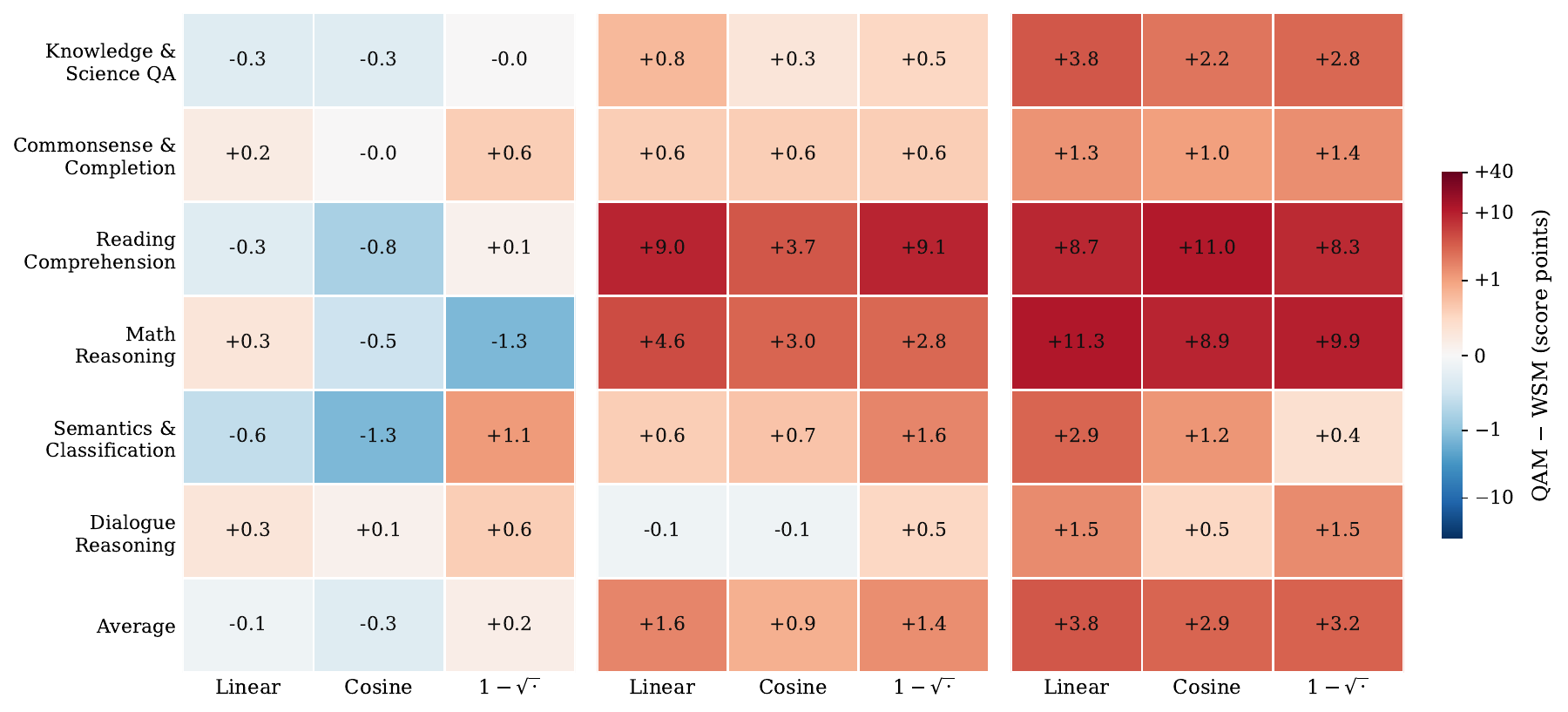}
    \caption{Paired score differences (\QAM{} minus \WSM{}) on
SmolLM3 across six capability groups and profiles. Average weights tasks by evaluation-set size.}
    \label{fig:heatmap:smollm3}
\end{figure}
\begin{figure}[htbp]
    \centering
    \includegraphics[width=0.95\textwidth]{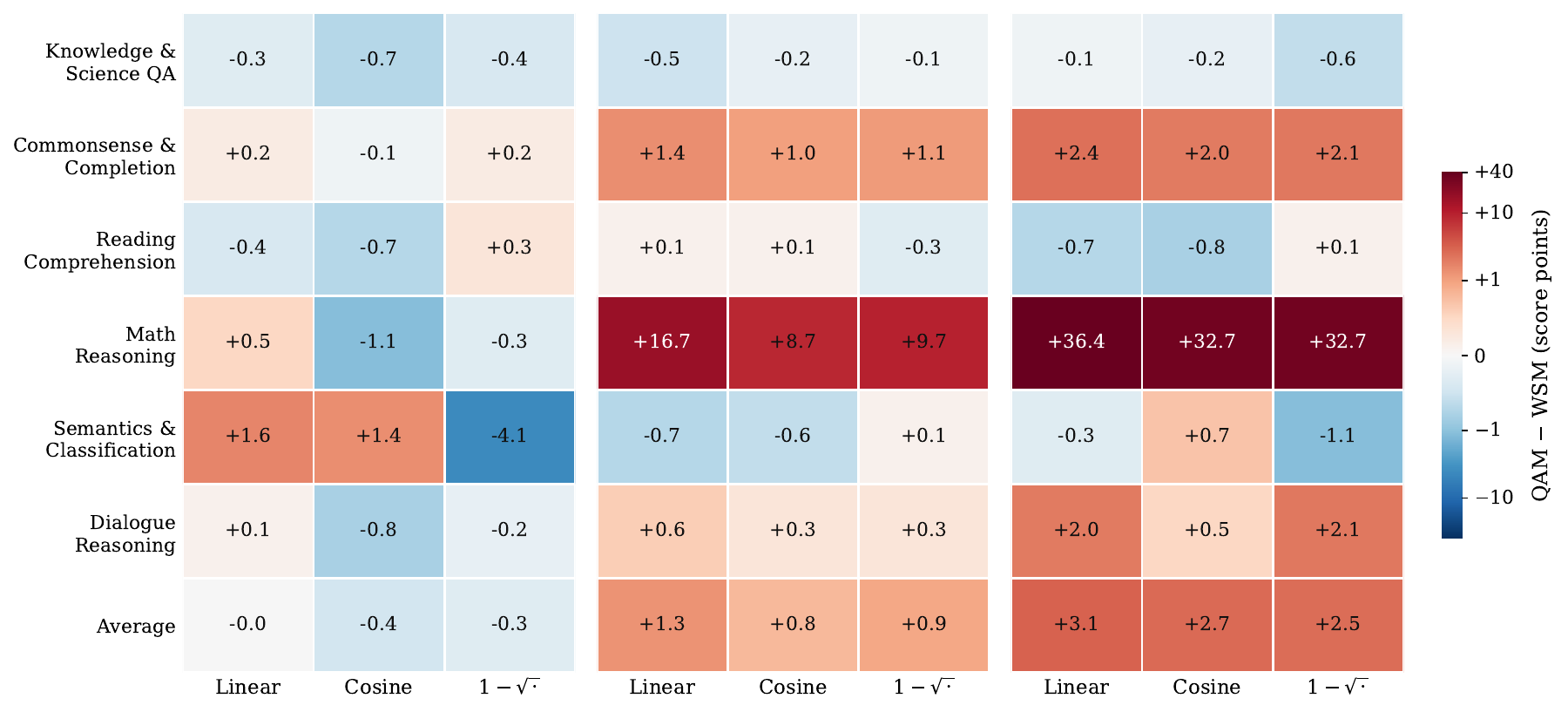}
    \caption{Same as Figure~\ref{fig:heatmap:smollm3}, but for OpenEuroLLM Prelude.}
    \label{fig:heatmap:prelude}
\end{figure}

\paragraph{Intermediate and long windows.}
\label{sec:intermediate-window}
\label{sec:long-window}
SmolLM3 improves in 16/18 intermediate-window and all 18 long-window
comparisons. Prelude improves in 12/18 and 11/18, respectively.
The largest long-window gains occur in Math: $+8.91$--$+11.30$ points
for SmolLM3 and $+32.68$--$+36.43$ for Prelude. Prelude's seven
long-window decreases are all below 1.2 points and occur in Knowledge,
Reading, or Semantics. Its large Math gap includes a WSM generation
failure, and Appendix~\ref{app:prelude_output_diagnostics} distinguishes
this behavior from the less severe SmolLM3 case.

\subsection{Comparison with individual checkpoints}
\label{sec:individual-checkpoint-comparison}

We compare QAM with individual checkpoints across all three windows
and all three profiles on each trajectory.
Table~\ref{tab:checkpoint-direct} summarizes comparisons with every
constituent checkpoint, the latest checkpoint, and the task-wise
best checkpoint within each window.

\begin{table}[htbp]
\centering
\small
\caption{QAM versus individual checkpoints across all nine configurations
per trajectory. Entries are wins/ties/losses.
``All'' counts comparisons with every constituent checkpoint.
``Latest'' and ``Best'' each contain 135 task--configuration comparisons
per trajectory.}
\label{tab:checkpoint-direct}
\begin{tabular}{@{}lrrr@{}}
\toprule
Trajectory & All & Latest & Best \\
\midrule
SmolLM3 & 1092/48/210 & 102/5/28 & 70/3/62 \\
Prelude & 2716/121/313 & 117/9/9 & 77/2/56 \\
\midrule
Total & 3808/169/523 & 219/14/37 & 147/5/118 \\
\bottomrule
\end{tabular}
\end{table}

QAM exceeds the latest checkpoint in 219 of 270 comparisons
and the task-wise best constituent in 147 of 270.
It wins 80.9\% of constituent-checkpoint comparisons on SmolLM3
and 86.2\% on Prelude.
Appendix~\ref{app:checkpoint-direct} reports the results
for every window, profile, and task.

\begin{remark}
    WSC single-checkpoint scores fluctuate substantially, and the best
checkpoint can exceed QAM by a large margin
(Appendix~\ref{app:full results}). This unfavorable case is retained
in all counts and aggregate comparisons.
\end{remark}

\subsection{Comparison with the best tested WSM}
\label{sec:cross-window-comparison}

We focus on intermediate and long windows, where the paired
comparisons show the clearest QAM gains. We examine whether
these gains persist when WSM can choose its best window and
profile from the full tested grid.

For each task, we compare the best QAM score across these two
windows and all three profiles with the best WSM score across
all nine tested window/profile combinations.
QAM exceeds this baseline on 8 of 15 tasks for SmolLM3 and
9 of 15 for Prelude.
These are task-wise maxima, and the maximizing configurations
can differ across tasks.
Appendix~\ref{app:best_observed_tasks} reports the full
task-wise comparison.

Table~\ref{tab:fixed_linear_math} further examines GSM8K with
QAM's profile fixed to linear, while allowing WSM to choose
its best window and profile.
QAM exceeds this baseline at both intermediate and long
windows on both trajectories, by 1.55--2.12 points.

\begin{table}[htbp]
\centering
\small
\caption{GSM8K with a fixed linear QAM profile versus the best
WSM score over all nine tested window/profile combinations
on each trajectory. Scores average flexible and strict extraction.}
\label{tab:fixed_linear_math}
\begin{tabular}{@{}llrrr@{}}
\toprule
Trajectory & QAM window & Linear QAM & Best WSM & $\Delta$ \\
\midrule
SmolLM3 & Last-10 & 43.71 & 41.74 & +1.97 \\
SmolLM3 & Last-15 & 43.29 & 41.74 & +1.55 \\
Prelude & Last-20 & 44.47 & 42.91 & +1.56 \\
Prelude & Last-40 & 45.03 & 42.91 & +2.12 \\
\bottomrule
\end{tabular}
\end{table}

\section{Second-Order Description and Coefficient geometry}
\label{sec:analysis}
\label{sec:degradation-diagnostic}

The variance identity relates the local discrepancy to coefficient
geometry. We examine this relation empirically on long-window GSM8K,
where WSM and QAM differ most. For paired WSM and QAM weights, let
$p^{(\lambda)}=(1-\lambda)c+\lambda q$, $\lambda\in[0,1]$.
The means of $c$ and $q$ agree, so
\begin{equation}
\operatorname{Var}(I_\lambda)
=\operatorname{Var}(I_c)-2\lambda\Delta_2(W).
\label{eq:interpolation-variance}
\end{equation}
Under Assumption~\ref{ass:common-map}, the same path satisfies
\begin{equation}
\theta_{\lambda}-\phi_k^W
=h^2(1-\lambda)\Delta_2(W)J_0f_0+O(h^3).
\label{eq:path-discrepancy}
\end{equation}
Thus, it reduces both the index variance and the model's leading
discrepancy coefficient. The path also changes higher moments and
interpolates between the two merged parameter vectors.

\begin{figure}[H]
    \centering
    \includegraphics[width=0.8\textwidth]{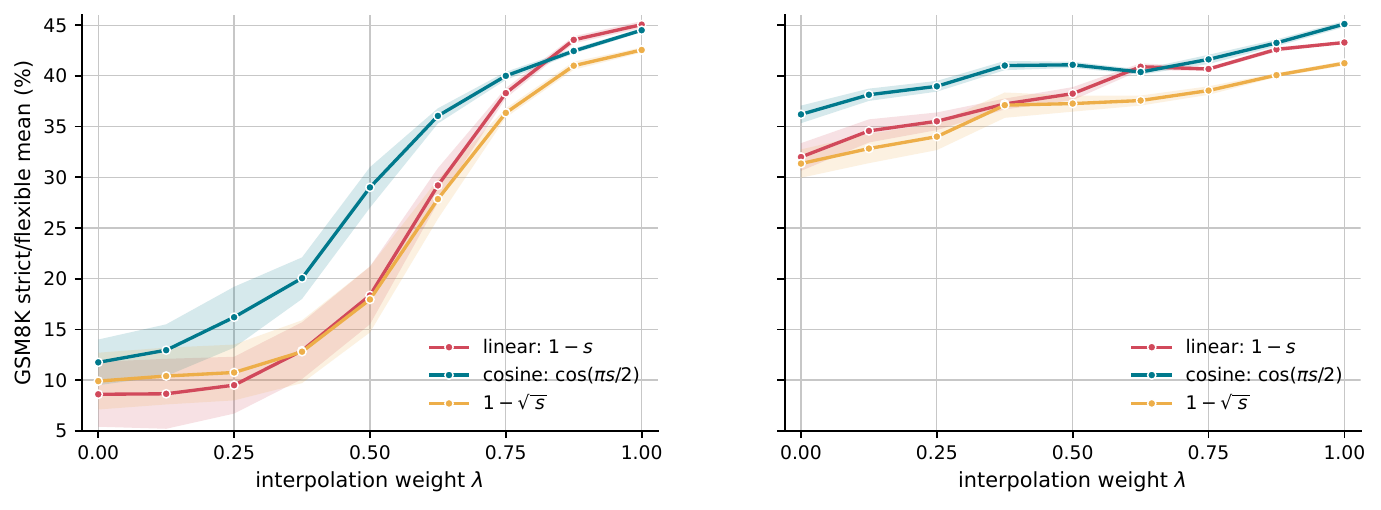}
    \caption{WSM--QAM interpolation on Prelude Last-40 (left) and
SmolLM3 Last-15 (right). Lines average flexible and strict GSM8K
extraction; shading spans the two extraction scores.}
    \label{fig:pathprelude}
\end{figure}

Figure~\ref{fig:pathprelude} shows overall improvement toward QAM
for all three profiles on both histories. Prelude's path also recovers
normal generation and termination behavior, whereas the SmolLM3 WSM
endpoints show much less severe formatting problems
(Appendix~\ref{app:prelude_output_diagnostics}). The trend of performance improvement is shared on both models.

To test whether the two index moments determine these scores, we pair
each $\lambda\in\{0.25,0.50,0.75\}$ mixture with a maximum-entropy (MaxEnt)
distribution on the same checkpoint grid. Its coefficients have the
form $r_i\propto\exp(a i+b i^2)$, with $a,b$ set by the mixture's
mean and variance. This fixes a second coefficient vector without
choosing it by benchmark performance. Appendix~\ref{app:offpath_controls}
specifies the construction.
Each pair therefore shares the same second-order expansion under the
local model, including its residual relative to the fixed reference.

Table~\ref{tab:offpath_controls} gives the paired scores.
The rankings vary: on SmolLM3, MaxEnt has higher flexible-extraction
scores in six of nine pairs, whereas on Prelude the mixture has higher
scores in seven of nine. At Prelude $\lambda=0.75$, the flexible-score
gaps favoring the mixture are 6.6, 1.1, and 13.4 points for the three
profiles. For these pairs, the input grid, center, variance, and second-order
reference discrepancy are all held fixed. Averaging the two extraction scores, the largest gap is 15.15 points.

\begin{table}[H]
\centering
\footnotesize
\setlength{\tabcolsep}{3.3pt}
\caption{GSM8K intermediate-target controls (\%). Mix stands for
$p^{(\lambda)}$. MaxEnt matches its two moments on the same grid.}
\label{tab:smol_offpath_controls}
\label{tab:offpath_controls}
\begin{tabular}{@{}lcrrrrrrrr@{}}
\toprule
 & & \multicolumn{4}{c}{SmolLM3 Last-15}
 & \multicolumn{4}{c}{Prelude Last-40} \\
\cmidrule(lr){3-6}\cmidrule(l){7-10}
 & & \multicolumn{2}{c}{Flex} & \multicolumn{2}{c}{Strict}
 & \multicolumn{2}{c}{Flex} & \multicolumn{2}{c}{Strict} \\
\cmidrule(lr){3-4}\cmidrule(lr){5-6}\cmidrule(lr){7-8}\cmidrule(l){9-10}
Profile & $\lambda$ & Mix & MaxEnt & Mix & MaxEnt & Mix & MaxEnt & Mix & MaxEnt \\
\midrule
Linear & 0.25 & 36.39 & 36.85 & 34.65 & 35.18 & 12.3 & 12.5 & 6.7 & 6.9 \\
 & 0.50 & 38.89 & 39.27 & 37.60 & 38.36 & 21.2 & 18.1 & 15.5 & 11.5 \\
 & 0.75 & 40.86 & 39.20 & 40.49 & 38.67 & 38.8 & 32.2 & 37.8 & 29.6 \\
\addlinespace[2pt]
Cosine & 0.25 & 39.50 & 40.33 & 38.44 & 39.73 & 19.2 & 18.4 & 13.2 & 13.6 \\
 & 0.50 & 41.39 & 42.00 & 40.79 & 40.79 & 31.0 & 27.4 & 27.0 & 23.0 \\
 & 0.75 & 42.08 & 42.15 & 41.17 & 41.02 & 40.4 & 39.3 & 39.6 & 38.4 \\
\addlinespace[2pt]
$1-\sqrt{\cdot}$ & 0.25 & 35.33 & 34.27 & 32.68 & 32.07 & 13.5 & 14.1 & 8.0 & 8.9 \\
 & 0.50 & 38.06 & 37.30 & 36.47 & 35.71 & 21.2 & 14.9 & 14.7 & 8.7 \\
 & 0.75 & 38.97 & 39.35 & 38.13 & 38.21 & 36.8 & 23.4 & 35.9 & 19.0 \\
\bottomrule
\end{tabular}
\end{table}

Separate MaxEnt controls matching the linear QAM endpoint's mean and
variance score slightly higher under both extraction rules. Their
mean-score advantages are 0.34 points on SmolLM3 and 0.68 on Prelude (Appendix~\ref{app:endpoint-maxent}).
These point estimates establish neither a significant difference nor
equivalence. They provide no empirical advantage for QAM over this
same-moment alternative. Quadratic exactness selects a complete rule
without guaranteeing a downstream ranking.

The index-variance view suggests a possible explanation for part of
QAM's long-window advantage. Compared with WSM, QAM concentrates the
merge weights into a narrower but still nondegenerate region of the
checkpoint window. This should not be interpreted as a causal claim
that lower variance is always better. Individual checkpoints have zero
index variance yet are often outperformed
(Table~\ref{tab:checkpoint-direct}), and the matched-moment controls
above show that the first two index moments do not fully determine downstream
performance.
\section{Discussion and Limitations}
\label{sec:discussion}

We design the merging coefficients based on a common transition model, while sparse Adam histories need not satisfy those assumptions.
The minimax result concerns parameter reconstruction, fixed $k$, and
$h\to0$ over a specified loss class. It gives neither optimal error
constants nor a downstream ordering or guarantee of reproducing decay
training. WSM's original evaluation
includes the \textbf{Warmup-Stable-Decay} (WSD) baseline, whereas a matched WSD comparison here would require additional training.

The empirical results cover two histories with shared checkpoints and
test examples. Task-wise maxima are
retrospective rather than one model's simultaneous scores.
Window length, spacing, and kernel shape are not
fully separated as changing merging granularity also changes QAM's effective width even on a fixed token interval. Prelude's severe generation failure should not be generalized to all WSM merges.
\section{Conclusion}
A saved history determines second-order reference behavior but leaves
third-order ambiguity, even for globally smooth strongly convex losses
and unrestricted reconstruction algorithms. QAM attains the sharp
uniform local rate and is uniquely selected among profile-based linear
rules by exact quadratic matching. Its sparse-Adam use yields practical
gains on two pretraining histories, especially at longer windows.
Matched-moment diagnostics show why consistency and downstream quality
remain distinct questions.

An important open question is to what extent this reconstruction viewpoint can approximate the reference endpoint of an actual learning-rate decay schedule,
rather than merely provide a principled criterion for coefficient design.
Stronger local models of the training dynamics may be needed to make this connection more faithful.

\label{main-text-end}

\clearpage
\appendix
\section{Proofs and Additional Analysis}
\label{app:proofs}

\subsection{Proof of local sequential consistency}
\label{app:common-proof}
\begin{proof}[Proof of Theorem~\ref{thm:common-consistency}]
The recorded checkpoints and the reference expand as
\begin{align}
\theta_i&=\theta_0+ih f_0+h^2\left[i a_0+\binom i2J_0f_0\right]+O(h^3),
\label{eq:common-recorded}\\
\phi_k^W&=\theta_0+hA_W f_0+h^2[A_W a_0+B_WJ_0f_0]+O(h^3).
\label{eq:common-reference-expansion}
\end{align}

    For fixed $k$ and sufficiently small $h$, bounded $f,a$ and the uniform
remainder keep both finite trajectories $O(h)$ from $\theta_0$ within
the assumed neighborhood. Taylor expansion gives
$f(\theta_0+u)=f_0+J_0u+O(\|u\|^2)$ and
$a(\theta_0+u)=a_0+O(\|u\|)$.
Induction on $i$ gives Eq.~\eqref{eq:common-recorded} as the coefficient
of $J_0f_0$ satisfies $C_{i+1}=C_i+i$, $C_0=0$, hence
$C_i=\binom i2$. Similarly,
$A_{j+1}=A_j+W_j$ and $B_{j+1}=B_j+W_jA_j$ give
Eq.~\eqref{eq:common-reference-expansion}.
For any fixed convex $p$, writing
$\mu_p=\sum_i ip_i$ and $\nu_p=\sum_i\binom i2p_i$ therefore yields
\begin{equation}
\sum_i p_i\theta_i-\phi_k^W
=h(\mu_p-A_W)f_0+h^2\bigl[(\mu_p-A_W)a_0
+(\nu_p-B_W)J_0f_0\bigr]+O(h^3).
\label{eq:common-general}
\end{equation}
This proves sufficiency in Theorem~\ref{thm:common-consistency}.
For necessity, consider the one-dimensional maps $F_h(x)=x+h$ at $x_0=0$
and then $F_h(x)=(1+h)x$ at $x_0=1$. The first forces
$\mu_p=A_W$ and the second forces $\nu_p=B_W$.
The identity $2\nu_p=\operatorname{Var}_p(I)+\mu_p^2-\mu_p$
gives the equivalent mean-and-variance condition.
The base tail sums give $\mu_c=A_W$ and
$\nu_c=\sum_j jW_j=\sum_{\ell<j}W_j$. Substituting these in
Eq.~\eqref{eq:common-general} proves Eq.~\eqref{eq:common-discrepancy}.
For QAM, $Q_W'(1)=A_W$ and $Q_W''(1)=2B_W$.
\end{proof}

\begin{remark}
The common-map assumption is an idealized model. It includes a finite block
of a fixed smooth deterministic update admitting this expansion, but
does not assert that a sparse Adam history obeys such a map in its
parameter state alone. Neither $h$ nor the expansion's remainder is
estimated from the experimental histories.
\end{remark}

\subsection{Proof of the Checkpoint-Information Bound}
\label{app:information-bound}

Recall that
\(\mathcal L_M\), defined in Eq.~\eqref{eq:loss-class}, consists of
\(C^4\) losses on \(\mathbb R^2\) satisfying
\[
\tfrac12 I\preceq\nabla^2L\preceq\tfrac52 I,
\qquad
\|\nabla L(0)\|\le 3,
\qquad
\|D^3L\|_\infty,\|D^4L\|_\infty\le M.
\]
Fix \(M>0\), \(k\ge2\), and \(W\in[0,1]^k\), with
\[
S_W=\sum_{\ell<j}W_jW_\ell(1-W_\ell)>0.
\]
The ordinary and reference trajectories start at the origin and satisfy
\[
\theta_{i+1}=\theta_i-h\nabla L(\theta_i),
\qquad
\phi_{j+1}=\phi_j-hW_j\nabla L(\phi_j),
\qquad
\theta_0=\phi_0=0.
\]
The observation is \(T_h(L)=(\theta_0,\ldots,\theta_k)\), and QAM is defined as
\[
Q_h(L)=\sum_{i=0}^k q_i\theta_i,
\qquad
\sum_{i=0}^kq_it^i
=\prod_{j=0}^{k-1}(1-W_j+W_jt).
\]

\begin{proof}[Proof of Theorem~\ref{thm:information-bound}]
We first bound the error of QAM uniformly over \(\mathcal L_M\).
We then construct two losses that give exactly the same observations
but whose reference endpoints differ at order \(h^3\).

\paragraph{Uniform upper bound.}
Write \(f=-\nabla L\). The assumptions give
\[
\|f(0)\|\le3,
\qquad
\|Df\|_\infty\le\tfrac52,
\qquad
\|D^2f\|_\infty,\|D^3f\|_\infty\le M.
\]
For an ordinary or reference update with weight \(w\in[0,1]\),
\[
\|z+hwf(z)\|
\le (1+\tfrac52h)\|z\|+3h.
\]
Thus, starting from zero and iterating at most \(k\) times yields
\[
\|z\|\le 3kh\exp(5kh/2).
\]
Thus every state is \(O(h)\), uniformly over \(L\in\mathcal L_M\)
and \(W\in[0,1]^k\).

Let
\[
a=f(0),\qquad B=Df(0),\qquad H=D^2f(0).
\]
Taylor expansion yields
\[
f(z)=a+Bz+\tfrac12H[z,z]+O(\|z\|^3),
\]
with a uniform remainder bound. Since each state is \(O(h)\),
the remainder contributes \(O(h^4)\) to one update.

For the ordinary trajectory, write
\[
\theta_i
=ih a+h^2c_iBa
+h^3\bigl[d_iB^2a+v_iH[a,a]\bigr]+O(h^4).
\]
Substituting this expression into
\(\theta_{i+1}=\theta_i+hf(\theta_i)\) gives
\[
c_{i+1}=c_i+i,
\qquad
d_{i+1}=d_i+c_i,
\qquad
v_{i+1}=v_i+\tfrac12i^2,
\]
with \(c_0=d_0=v_0=0\). Hence
\[
c_i=\binom i2,
\qquad
d_i=\binom i3,
\qquad
v_i=\binom i3+\tfrac12\binom i2.
\]
It follows that
\[
\theta_i
=ih a+h^2\binom i2Ba
+h^3\left[
\binom i3B^2a+
\left(\binom i3+\tfrac12\binom i2\right)H[a,a]
\right]+O(h^4).
\]

For the reference trajectory, define
\[
A_j=\sum_{\ell<j}W_\ell,
\qquad
E_{2,j}=\sum_{a<b<j}W_aW_b,
\qquad
E_{3,j}=\sum_{a<b<c<j}W_aW_bW_c.
\]
The same substitution into the reference trajectory gives
\[
\phi_j
=hA_ja+h^2E_{2,j}Ba
+h^3\bigl[E_{3,j}B^2a+R_jH[a,a]\bigr]+O(h^4),
\]
where the coefficients satisfy
\[
\begin{aligned}
A_{j+1}&=A_j+W_j,\\
E_{2,j+1}&=E_{2,j}+W_jA_j,\\
E_{3,j+1}&=E_{3,j}+W_jE_{2,j},\\
R_{j+1}&=R_j+\tfrac12W_jA_j^2.
\end{aligned}
\]
All four coefficients start at zero. If \(e_m\) denotes the
\(m\)-th elementary symmetric polynomial in \(W_0,\ldots,W_{k-1}\),
then
\[
\phi_k
=he_1a+h^2e_2Ba
+h^3\left[
e_3B^2a+\tfrac12\sum_jW_jA_j^2H[a,a]
\right]+O(h^4).
\]
All remainder bounds above are uniform because \(k\) is fixed and
the derivative bounds hold throughout the class.

To compute the corresponding coefficients for QAM, expand its
generating polynomial at \(t=1\):
\[
\sum_iq_i(1+s)^i=\prod_j(1+W_js).
\]
Comparing coefficients of \(s^m\) gives
\[
\sum_iq_i\binom im=e_m.
\]
Consequently,
\[
Q_h(L)
=he_1a+h^2e_2Ba
+h^3\left[
e_3B^2a+\left(e_3+\tfrac12e_2\right)H[a,a]
\right]+O(h^4).
\]
Also,
\[
\begin{aligned}
\sum_jW_jA_j^2
&=\sum_jW_j\left(
\sum_{\ell<j}W_\ell^2
+2\sum_{a<b<j}W_aW_b
\right)\\
&=\sum_{\ell<j}W_jW_\ell^2+2e_3.
\end{aligned}
\]
Subtracting the reference expansion therefore yields
\[
Q_h(L)-\phi_k(L)
=\tfrac12h^3S_WH[a,a]+O_{k,M}(h^4).
\]
Since \(\|H[a,a]\|\le9M\), this proves the uniform \(O(h^3)\)
upper bound.

\textbf{Two losses with the same checkpoint record.}
The lower bound uses a perturbation whose value and gradient vanish
on the ordinary trajectory, but its gradient does not vanish
on the reference trajectory.

Consider
\[
r(x,y)=y-x^2-2x.
\]
Choose a fixed smooth cutoff
\(\chi\in C_c^\infty(\mathbb R^2)\) that equals one on the unit
ball and zero outside the ball of radius two. Set
\[
P=\tfrac12\chi r^2,
\qquad
K=\max\left\{
1,\|D^2P\|_\infty,\|D^3P\|_\infty,\|D^4P\|_\infty
\right\},
\qquad
\varepsilon=\frac{\min\{1/2,M\}}{K}.
\]
Both \(P\) and \(\varepsilon>0\) are independent of \(h\). \(K\) is evidently finite since \(P\) is compactly supported.

For \(0\le h\le1/2\), define
\[
L_h^{(0)}(x,y)
=\tfrac12(1+x)^2+\tfrac12(2-h)(1+y)^2,
\qquad
L_h^{(1)}=L_h^{(0)}+\varepsilon P.
\]
We check that both losses belong to the same fixed class
\(\mathcal L_M\). The Hessian of \(L_h^{(0)}\) is
\(\operatorname{diag}(1,2-h)\), which lies between \(I\) and \(2I\).
Moreover,
\[
\|\varepsilon D^2P\|_\infty\le\tfrac12,
\qquad
\|\varepsilon D^3P\|_\infty,
\|\varepsilon D^4P\|_\infty\le M.
\]
These bounds give the required Hessian and higher-derivative bounds
for \(L_h^{(1)}\) everywhere on \(\mathbb R^2\).
Finally, since \(\nabla P(0)=0\), we have
\[
\|\nabla L_h^{(b)}(0)\|
=\sqrt{1+(2-h)^2}<3,
\qquad b\in\{0,1\}.
\]

Ordinary GD on \(L_h^{(0)}\) gives the exact trajectory
\[
\theta_i=\bigl((1-h)^i-1,\,(1-h)^{2i}-1\bigr).
\]
Thus \(r(\theta_i)=0\) for every \(i\). On the entire set
\(\{r=0\}\), we have
\[
P=0,
\qquad
\nabla P=\tfrac12r^2\nabla\chi+\chi r\nabla r=0.
\]
The two losses therefore have the same gradient at every ordinary
iterate. Starting from the same initial state, they generate exactly
the same checkpoint record by induction. Their loss values and
gradients also agree at every recorded checkpoint, including
\(\theta_k\).

\textbf{Separation of the reference endpoints.}
The uniform state bound proved above applies to both losses.
For sufficiently small \(h\), both reference trajectories lie in the
unit ball, where \(\chi=1\). Their negative gradients there are
\[
f_h(x,y)=\bigl(-(1+x),\,(-2+h)(1+y)\bigr),
\qquad
g_h=f_h-\varepsilon G,
\qquad
G=r\nabla r.
\]

First consider the reference trajectory for \(L_h^{(0)}\).
Write
\[
z_j=\phi_j(L_h^{(0)})=(x_j,y_j),
\qquad
X_j=1+x_j,
\qquad
Y_j=1+y_j.
\]
The updates give
\[
X_{j+1}=(1-hW_j)X_j,
\qquad
Y_{j+1}=(1-2hW_j+h^2W_j)Y_j.
\]
Since \(r(z_j)=Y_j-X_j^2\), the residual
\(r_j=r(z_j)\) satisfies the exact recurrence
\[
r_{j+1}
=(1-2hW_j+h^2W_j)r_j
+h^2W_j(1-W_j)X_j^2.
\]
Here \(X_j=1+O(h)\) and \(r_0=0\). Induction yields
\[
r_j=h^2V_j+O(h^3),
\qquad
V_j=\sum_{\ell<j}W_\ell(1-W_\ell).
\]
Also noting that,
\[
\nabla r(z_j)=(-2X_j,1)=(-2,1)+O(h),
\]
and hence
\[
G(z_j)=h^2V_j(-2,1)+O(h^3).
\]

Now let
\[
\delta_j=\phi_j(L_h^{(1)})-\phi_j(L_h^{(0)}),
\qquad
D_h=\operatorname{diag}(-1,-2+h).
\]
Since \(f_h\) is affine, subtracting the two reference updates gives
the exact recurrence
\[
\delta_{j+1}
=(I+hW_jD_h)\delta_j
-\varepsilon hW_jG(z_j+\delta_j),
\qquad
\delta_0=0.
\]
Both trajectories lie in the unit ball, where \(G\) has a fixed
Lipschitz bound. Since \(G(z_j)=O(h^2)\), this recurrence implies
\[
\|\delta_{j+1}\|
\le (1+Ch)\|\delta_j\|+Ch^3
\]
for a constant \(C\) independent of \(h\). Iterating for the fixed
number \(k\) of steps gives \(\delta_j=O(h^3)\).

We can now identify the leading term. The bound on \(\delta_j\) gives
\[
hW_jD_h\delta_j=O(h^4),
\qquad
hW_j\bigl[G(z_j+\delta_j)-G(z_j)\bigr]=O(h^4).
\]
Thus the difference recurrence reduces to
\[
\begin{aligned}
\delta_{j+1}
&=\delta_j-\varepsilon hW_jG(z_j)+O(h^4)\\
&=\delta_j+\varepsilon h^3W_jV_j(2,-1)+O(h^4).
\end{aligned}
\]
Summing over \(j\), and using
\[
\sum_jW_jV_j
=\sum_{\ell<j}W_jW_\ell(1-W_\ell)=S_W,
\]
we obtain
\[
\phi_k(L_h^{(1)})-\phi_k(L_h^{(0)})
=\varepsilon S_Wh^3(2,-1)+O(h^4).
\]
Throughout computations, every remainder constant is independent of \(h\), and the number of
steps is fixed.

\textbf{From indistinguishability to the lower bound.}
Any deterministic algorithm returns the same output on the two
identical records. By the triangle inequality, at least one of
its two errors is at least half the distance between the targets.
Therefore,
\[
\begin{aligned}
\sup_{L\in\mathcal L_M}
\|A(T_h(L),h,W)-\phi_k(L)\|
&\ge
\tfrac12\|\phi_k(L_h^{(1)})-\phi_k(L_h^{(0)})\|\\
&\ge
\tfrac{\varepsilon\sqrt5\,S_W}{2}h^3-Ch^4.
\end{aligned}
\]
Taking the infimum over algorithms proves the lower bound.
For example, after reducing \(h_0\), one may take
\(c=\varepsilon\sqrt5\,S_W/4>0\).
The same argument applies when checkpoint losses and gradients
are included in the record, since those observations also agree.

For a randomized algorithm, the common record induces the same
output distribution for both losses. If \(Z\) has that distribution,
then
\[
\max_{b\in\{0,1\}}
\mathbb E\|Z-\phi_k(L_h^{(b)})\|
\ge
\tfrac12\|\phi_k(L_h^{(1)})-\phi_k(L_h^{(0)})\|.
\]
Thus the lower bound also holds for worst-case expected norm error.
\end{proof}

\begin{remark}
    The loss class \(\mathcal L_M\) is fixed independently of \(h\).
For each \(h\), the supremum over that class may select a different
pair \(L_h^{(0)},L_h^{(1)}\). The theorem therefore gives a uniform
worst-case lower bound, rather than a pointwise lower bound for a
single fixed loss.

If \(S_W=0\), every positive weight that precedes another positive
weight must equal one. After zero weights are removed, the profile
consists of \(m\) ones followed by at most one fractional weight
\(\alpha\in(0,1)\). With no fractional weight, the reference equals
\(\theta_m\). Otherwise, it equals
\[
(1-\alpha)\theta_m+\alpha\theta_{m+1}.
\]
QAM reconstructs these references exactly. The all-zero profile
returns \(\theta_0=0\).

If the derivative bound is instead set to \(M=0\), every loss in
the class is quadratic and QAM is exact.
\end{remark}

\begin{remark}
    The theorem holds for fixed \(k,W\), and \(M>0\) as \(h\to0\).
Allowing these quantities to vary with \(h\) requires separate
control of the constants. The conclusion concerns the optimal
order of parameter reconstruction error. It does not establish
an optimal leading constant, a downstream performance ordering,
or a guarantee of reproducing decay training.
\end{remark}

\subsection{The sharp nonlinear limit of fixed checkpoint weights}
\label{app:order-barrier}

For fixed Euler steps, a more specific coefficient obstruction is:

\begin{theorem}[Nonlinear order barrier]
\label{thm:order-barrier}
Fix $k,W$ and let $f$ range over $C^3$ vector fields with bounded local
derivatives. If $S_W>0$, no real coefficients $p_i(W)$ independent of
$h,f$ and the initial state satisfy
$\sum_i p_i\theta_i-\phi_k^W=O(h^4)$ for all such Euler trajectories.
QAM has the leading remainder
\begin{equation}
\theta_{\mathrm{QAM}}-\phi_k^W
=\tfrac12h^3 S_W D^2f(\theta_0)[f(\theta_0),f(\theta_0)]+O(h^4).
\label{eq:qam-third-remainder}
\end{equation}
If $S_W=0$, QAM matches the reference exactly for every $f$.
\end{theorem}

The following calculation uses the classical elementary-differential
expansion underlying numerical order conditions~\citep{butcher1963coefficients}.
The target here is a prescribed sequence of relaxed Euler updates,
rather than the exact flow of a differential equation. Fix $k,W,x$,
let $f\in C^3$ on a neighborhood of $x$, with bounded derivatives there,
and set $f_0=f(x)$, $J=Df(x)$, $K=D^2f(x)$.
For sufficiently small $h$, both finite trajectories remain in that
neighborhood. Write
\[
C_W=\sum_{a<b<c}W_aW_bW_c,\quad
D_W=\sum_{\ell<j}W_jW_\ell^2,\quad
S_W=B_W-D_W=\sum_{\ell<j}W_jW_\ell(1-W_\ell).
\]
Taylor expansion gives the following third-order formulas:
\begin{align}
\theta_i={}&x+ihf_0+h^2\binom i2Jf_0\notag\\
&+h^3\left[\binom i3J^2f_0+
 \left(\binom i3+\tfrac12\binom i2\right)K[f_0,f_0]\right]+O(h^4),
\label{eq:euler-third}\\
\phi_k^W={}&x+hA_Wf_0+h^2B_WJf_0\notag\\
&+h^3\left[C_WJ^2f_0+
 \left(C_W+\tfrac12D_W\right)K[f_0,f_0]\right]+O(h^4).
\label{eq:reference-third}
\end{align}
To verify the first identity, the chain coefficient obeys
$u_{i+1}=u_i+\binom i2$ and hence $u_i=\binom i3$.
The branched coefficient obeys $v_{i+1}=v_i+i^2/2$ and hence
$v_i=i(i-1)(2i-1)/12=\binom i3+\tfrac12\binom i2$.
For the reference, the chain coefficient is
$\sum_j W_j B_j=C_W$. Its branched coefficient is
$\tfrac12\sum_jW_j A_j^2=C_W+D_W/2$, where
$A_j=\sum_{\ell<j}W_\ell$ and
$B_j=\sum_{a<b<j}W_aW_b$.

For any normalized real coefficients $p$ satisfying the two
second-order conditions, put $M_3(p)=\sum_i p_i\binom i3$.
Subtracting the preceding expansions yields
\begin{equation}
\sum_i p_i\theta_i-\phi_k^W
=h^3\left[(M_3(p)-C_W)J^2f_0+
\left(M_3(p)-C_W+\tfrac12S_W\right)K[f_0,f_0]\right]+O(h^4).
\label{eq:third-discrepancy}
\end{equation}
Thus a single free third factorial moment affects both elementary
differentials by the same amount. Their required values differ by
$S_W/2$.

\begin{proof}[Proof of Theorem~\ref{thm:order-barrier}]
Universal $O(h^4)$ matching first forces normalization and the two
lower-order moments. Normalization follows by taking $f=0$ and a
nonzero initial state. The scalar fields $f(x)=1$ and $f(x)=x$ force
the first and second moments; the latter, initialized at $x=1$,
also forces $M_3(p)=C_W$. Taking $f(x)=1+x^2/2$ at $x=0$ then gives
$J=0$ and $K[f_0,f_0]=1$, so Eq.~\eqref{eq:third-discrepancy}
forces $S_W=0$. This argument applies even to signed coefficients.
Conversely, the generating polynomial gives
$M_3(q)=Q_W'''(1)/6=C_W$, and substitution gives
Eq.~\eqref{eq:qam-third-remainder}.

Each summand in $S_W$ is nonnegative. Therefore $S_W=0$ if and only
if every positive entry that precedes a later positive entry is one.
After zero entries are removed, the profile is consequently a string
of $m$ ones followed by at most one fractional entry $w$.
The reference then takes $m$ ordinary steps and one final relaxation,
so $\phi_k^W=(1-w)\theta_m+w\theta_{m+1}$ for every $f$ and every
step size for which the iterates exist. QAM has precisely these two
coefficients; if no fractional entry occurs, it selects $\theta_m$.
For a nonincreasing profile this is exactly the pattern
$(1,\ldots,1,w,0,\ldots,0)$. This proves the sharp exception.
\end{proof}

As an exact two-update illustration, take $k=2$, $x=0$ and
$f(x)=1+x^2/2$. The ordinary states are $0,h,2h+h^3/2$, and
\[
\theta_{\rm QAM}-\phi_2^W
=\tfrac12h^3W_0W_1(1-W_0).
\]
The discrepancy already occurs for a one-dimensional gradient field,
since $f=-L'$ for $L(x)=-x-x^3/6$.
This calculation concerns fixed linear weights. Theorem~\ref{thm:information-bound}
establishes a stronger uniform lower bound for arbitrary reconstruction
algorithms with the same checkpoint information; additional off-trajectory
gradient queries change that information model.

\subsection{Consecutive GD as a specialization}
\label{app:gd-specialization}
For a fixed objective $L$, the ordinary and reference GD trajectories are
\begin{equation}
\theta_{i+1}=\theta_i-\eta\nabla L(\theta_i),
\qquad i=0,\ldots,k-1.
\label{eq:qam-local-map}
\end{equation}
\begin{equation}
\phi_0^W=\theta_0,\qquad
\phi_{j+1}^W=\phi_j^W-\eta W_j\nabla L(\phi_j^W),
\qquad j=0,\ldots,k-1.
\label{eq:qam-reference}
\end{equation}
\begin{assumption}[Fixed smooth objective and finite GD trajectories]
\label{ass:qam-local-dynamics}
The objective $L$ is fixed and belongs to $C^3$ on a neighborhood
of $\theta_0$, with locally bounded derivatives through order three.
Both trajectories follow Eqs.~\eqref{eq:qam-local-map}
and~\eqref{eq:qam-reference}. The initial state, $k$, and $W$
remain fixed as $\eta\to0$.
\end{assumption}
\begin{theorem}[Characterization of local second-order consistency]
\label{thm:QAM}
Fix $k\geq1$, $W\in[0,1]^k$, and a convex coefficient vector
$p=(p_0,\ldots,p_k)$ independent of $\eta$, $L$, and $\theta_0$.
The bound
\[
\sum_{i=0}^kp_i\theta_i-\phi_k^W=O(\eta^3)
\]
holds as $\eta\to0$ for every fixed objective and initial state
satisfying Assumption~\ref{ass:qam-local-dynamics} if and only if
\begin{equation}
\mathbb E_p[I]=A_W,\qquad
\mathbb E_p\!\left[\binom I2\right]=B_W.
\label{eq:qam-consistency-conditions}
\end{equation}
Here $\Pr_p(I=i)=p_i$. Equivalently, $p$ has the same checkpoint-index mean and variance
as $q$. In particular,
\begin{equation}
\theta_{\mathrm{QAM}}-\phi_k^W=O(\eta^3).
\label{eq:qam-local-error}
\end{equation}
\end{theorem}
This is the $h=\eta$, $f=-\nabla L$, $a=0$ case of
Theorem~\ref{thm:common-consistency}.

\subsection{Proof of Theorem~\ref{thm:qam-quadratic-exactness}}
\label{app:gd-derivation}

\subsubsection{Exactness by composing quadratic GD updates}
\label{app:quadratic-gd-construction}

\begin{proof}[Proof of exactness]
Fix $k\geq1$, $W\in[0,1]^k$, $\eta>0$, and a quadratic objective
as in Eq.~\eqref{eq:qam-quadratic-model}. Its gradient is
$\nabla L(\theta_0+u)=g_0+H_0u$ for every $u$, with arbitrary $g_0$
and symmetric $H_0$. In displacement coordinates
$u=\theta-\theta_0$, introduce
\begin{equation}
\mathcal M_\eta=
\begin{pmatrix}
I-\eta H_0&-\eta g_0\\
0&1
\end{pmatrix}.
\label{eq:qgd-augmented-matrix}
\end{equation}
Ordinary GD acts on $(u,1)^\top$ through $\mathcal M_\eta$;
a reference step with learning rate $\eta W_j$ acts through
\[
\begin{pmatrix}
I-\eta W_jH_0&-\eta W_jg_0\\
0&1
\end{pmatrix}
=(1-W_j)I+W_j\mathcal M_\eta.
\]
These matrices are polynomials in the same matrix and hence
commute. The generating polynomial in Eq.~\eqref{eq:qam-weights}
therefore gives the exact matrix identity
\begin{equation}
\prod_{j=0}^{k-1}\bigl[(1-W_j)I+W_j\mathcal M_\eta\bigr]
=Q_W(\mathcal M_\eta)=\sum_{i=0}^kq_i\mathcal M_\eta^i.
\label{eq:qgd-matrix-identity}
\end{equation}
Both trajectories start from $\theta_0$, so applying this identity
to $(0,1)^\top$ yields
\[
\begin{pmatrix}\phi_k^W-\theta_0\\1\end{pmatrix}
=\sum_{i=0}^kq_i
\begin{pmatrix}\theta_i-\theta_0\\1\end{pmatrix}.
\]
Since $\sum_iq_i=Q_W(1)=1$, the upper coordinates prove
Eq.~\eqref{eq:qam-quadratic-endpoint}. The term $-\eta g_0$
retains the linear part of the quadratic objective throughout
the calculation. No stationary point or inverse of $H_0$ is
needed: $H_0$ may be singular or indefinite. This finite
algebraic identity requires neither a small step size nor
convergence of the GD trajectory.
\end{proof}

\subsubsection{Uniqueness across quadratic objectives}
\label{app:qam-uniqueness}

\begin{proof}[Proof of uniqueness]
Fix $k$, $W$, and $\eta>0$, and suppose that the same real
coefficient vector $p(W)$ satisfies the endpoint identity for
every quadratic loss and initial state. Keep this vector fixed
while varying the one-dimensional convex losses
\[
L_\lambda(x)=\tfrac12\lambda x^2,
\qquad x_0=1,\qquad 0<\lambda<1/\eta.
\]
Writing $t=1-\eta\lambda$ gives $x_i=t^i$ and
\[
\phi_k^W=\prod_{j=0}^{k-1}(1-\eta W_j\lambda)
=\prod_{j=0}^{k-1}(1-W_j+W_jt)=Q_W(t).
\]
Because $t$ ranges over $(0,1)$, exact matching requires
\[
\sum_{i=0}^kp_i(W)t^i=Q_W(t)\qquad\text{for every }t\in(0,1).
\]
Two polynomials agreeing on this set have identical coefficients,
so $p_i(W)=q_i(W)$ for all $i$. No convexity or normalization
assumption on $p$ is needed. If the dimension is fixed above one,
the same scalar example embeds in a single coordinate.
\end{proof}

The uniqueness requirement is universal across quadratic
objectives.

\subsection{Proof of Theorem~\ref{thm:QAM} and the variance identity}
\label{A.1}

\begin{proof}[Proof of Theorem~\ref{thm:QAM}]
Under Assumption~\ref{ass:qam-local-dynamics}, the GD map
satisfies Assumption~\ref{ass:common-map} with
$h=\eta$, $f=-\nabla L$, and $a=0$.
The expansions in Appendix~\ref{app:common-proof} therefore
give sufficiency and QAM's local error bound.
The base-discrepancy formula follows from
Eq.~\eqref{eq:common-discrepancy} under the same substitution.

For necessity, $L(x)=x$ at $x_0=0$ forces the first moment
condition. With that condition satisfied, $L(x)=x^2/2$
at $x_0=1$ forces the second. These examples embed in one
coordinate in any higher dimension.
The equivalent mean-and-variance condition follows from
the moment identity in Appendix~\ref{app:common-proof}.
\end{proof}

\paragraph{Scope of the necessity statement.}
The necessity statement concerns one fixed coefficient rule across
objectives, and a particular trajectory may admit other matching weights.
As noted in Section~\ref{sec:qam-construction}, second-order consistency alone leaves a
$(k-2)$-dimensional family of convex solutions when $k\ge3$ and
$W_j\in(0,1)$. Theorem~\ref{thm:qam-quadratic-exactness} selects $q$
by the stronger requirement of exact matching across quadratics.

\begin{proof}[Proof of Eq.~\eqref{eq:qam-variance-gap}]
This calculation uses only the coefficients $c,q$ and their
common input profile $W$. For $I_c\sim c$ and $I_q\sim q$, the
coefficient sums in Appendix~\ref{app:common-proof} give the common mean
$\mathbb E[I_c]=\mathbb E[I_q]=A_W$. Consequently,
\[
\begin{aligned}
\operatorname{Var}(I_c)-\operatorname{Var}(I_q)
&=2(\nu_c-\nu_q)\\
&=2\sum_{\ell<j}(W_j-W_\ell W_j)
=2\Delta_2(W)\geq0.
\end{aligned}
\]
The inequality follows from $0\leq W_j\leq1$.
\end{proof}

\clearpage

\section{Evaluation Benchmarks}
\label{app:experimental-details}

Main-text Table~\ref{tab:settings} gives the full checkpoint ranges and storage
intervals. A Last-$n$ window uses the last $n$ chronological checkpoints
and therefore spans $n-1$ intervals.

The SmolLM3 checkpoint steps are $3{,}640{,}000+40{,}000i$,
$i=0,\ldots,14$.
The Prelude steps are $859{,}200+2{,}400i$, $i=0,\ldots,39$.
These progressions specify the complete checkpoint grids.

Table~\ref{tab:groups} specifies all 15 tasks. Accuracy tasks use
\texttt{lm\_eval v0.4.13}. MMLU and GSM8K use 5-shot evaluation, and the other accuracy
tasks use zero-shot evaluation. Capability-group means weight accuracy
tasks by evaluation-set size. Math averages the flexible and strict
GSM8K extraction scores. MuTual uses MRR.
The available matched-kernel protocols are in
Appendix~\ref{app:matched_second_moment_controls}.

\begin{table}[H]
\centering
\small
\caption{Capability groups.}
\label{tab:groups}
\begin{tabular}{@{}L{3.1cm}L{7.6cm}l@{}}
\toprule
Group & Tasks & Metric \\
\midrule
Knowledge & ARC-Easy, ARC-Challenge~\citep{clark2018arc}, OpenBookQA~\citep{mihaylov2018openbookqa}, SciQ~\citep{welbl2017sciq}, MMLU (5-shot)~\citep{hendrycks2021mmlu} & acc. \\
Commonsense & HellaSwag~\citep{zellers2019hellaswag}, PIQA~\citep{bisk2020piqa}, COPA~\citep{gordon2012copa}, WinoGrande~\citep{sakaguchi2020winogrande} & acc. \\
Reading & BoolQ~\citep{clark2019boolq} & acc. \\
Math & GSM8K flexible, GSM8K strict~\citep{cobbe2021training} & exact match \\
Semantics & SST-2~\citep{socher2013recursive}, WiC~\citep{pilehvar2019wic}, WSC~\citep{levesque2012winograd} & acc. \\
Dialogue & MuTual~\citep{cui2020mutual} & MRR \\
\bottomrule
\end{tabular}
\end{table}

For each trajectory, window, and profile, QAM and WSM merge the
same checkpoints. Merges are accumulated in \texttt{fp32} and exported
in \texttt{bf16}, with no additional training, calibration, or prompt
tuning.

\clearpage

\section{Checkpoint Weights}
\label{app:weights}

Figure~\ref{fig:coefficient-rules} shows the WSM and QAM weights
defined in Eqs.~\eqref{eq:inverse-tail-sum}
and~\eqref{eq:qam-weights} across the window lengths and profiles
used in our experiments.

\begin{figure}[H]
\centering
\includegraphics[width=\textwidth]{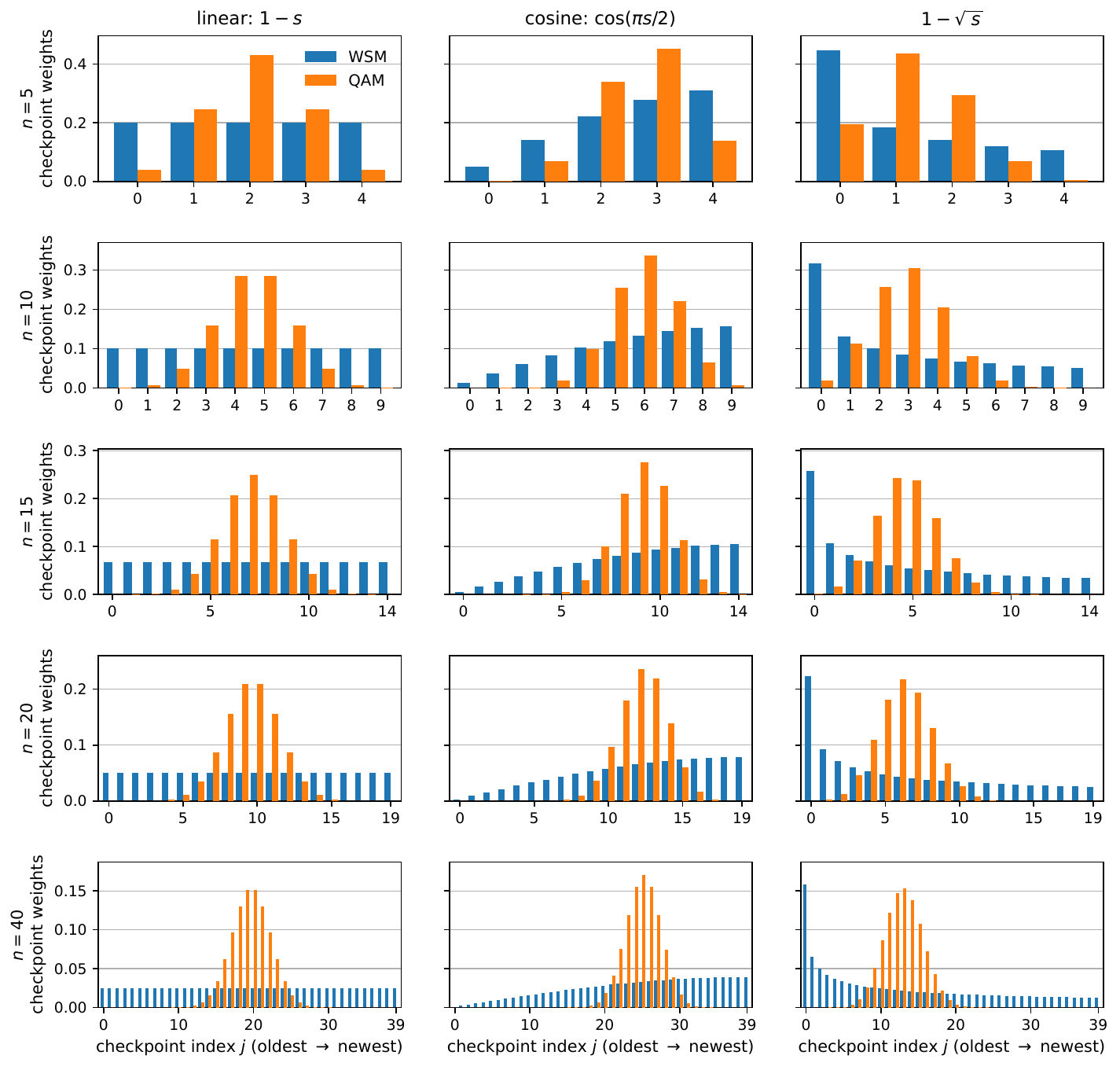}
\caption{WSM and QAM checkpoint weights.
Rows correspond to windows containing $n=5,10,15,20,40$
checkpoints, and columns correspond to the three prescribed
profiles. Checkpoint indices increase from oldest to newest.}
\label{fig:coefficient-rules}
\end{figure}

\clearpage

\section{Moment identities for the interpolation diagnostic}
\label{app:lambda-moments}

For a coefficient distribution $p$, let $\mu(p)$ and
$\mathrm{Var}(p)$ denote the mean and variance of its checkpoint
index. Appendix~\ref{app:common-proof} gives
$\mu(c)=\mu(q)=A_W$.

Consider the interpolation used in
Section~\ref{sec:degradation-diagnostic},
\[
p^{(\lambda)}=(1-\lambda)c+\lambda q,
\qquad 0\leq\lambda\leq1.
\]
Its mean is independent of $\lambda$:
\[
\mu(p^{(\lambda)})
=(1-\lambda)\mu(c)+\lambda\mu(q)
=A_W.
\]
Writing $\mu=A_W$, we obtain
\[
\begin{aligned}
\mathrm{Var}(p^{(\lambda)})
&=\sum_{i=0}^k(i-\mu)^2
   \bigl[(1-\lambda)c_i+\lambda q_i\bigr]\\
&=(1-\lambda)\mathrm{Var}(c)
  +\lambda\mathrm{Var}(q)\\
&=\mathrm{Var}(c)-2\lambda\Delta_2(W),
\end{aligned}
\]
where the last equality follows from
Eq.~\eqref{eq:qam-variance-gap}.
Thus the interpolation preserves the mean checkpoint index.
Its variance decreases linearly when $\Delta_2(W)>0$
and remains constant when $\Delta_2(W)=0$.

\clearpage
\section{Generation and Termination Diagnostics}
\label{app:prelude_output_diagnostics}

We report output-level diagnostics on Prelude Last-40 and
SmolLM3 Last-15. They describe generation behavior on the tested
models.

\subsection{Prelude Last-40}

We first examine the unusually low Prelude Last-40 \WSM{} GSM8K endpoint.

The \WSM{} endpoint exhibits abnormal generation behavior. Only a small fraction of outputs contain the expected final-answer delimiter, while many generations lose the normal line-by-line solution structure and enter repeated 10-gram patterns. The median output length is also much larger than at the paired \QAM{} endpoint, indicating that the model often continues generating instead of terminating with a concise final answer. These observations show that the degradation of WSM at long windows is expressed not only through final-answer correctness.

Recall the interpolation path defined in Section~\ref{sec:degradation-diagnostic}. As $\lambda$ increases toward the \QAM{} endpoint, these symptoms are progressively reduced. The generations recover normal answer formatting and line structure, repeated text becomes much less frequent, and the output length returns to the range typical of well-formed GSM8K solutions. The diagnostics in Table~\ref{tab:prelude_output_diagnostics} quantify these output-level symptoms along the interpolation path.

\begin{table}[H]
\centering
\small
\caption{Output diagnostics along the linear Prelude Last-40 interpolation path on GSM8K. The delimiter column records the fraction of generations containing the standard GSM8K answer delimiter. The no-newline column records outputs without a newline. The loop column records repeated 10-gram behavior. Median length is measured in characters.}
\label{tab:prelude_output_diagnostics}
\begin{tabular}{@{}lcccc@{}}
\toprule
Model & Delimiter & No newline & Loop & Median length \\
\midrule
$\lambda=0$     & 13.7 & 83.3 & 82.8 & 1027 \\
$\lambda=1/2$   & 34.3 & 57.2 & 60.0 & 840  \\
$\lambda=3/4$   & 79.1 & 11.0 & 14.8 & 374  \\
$\lambda=1$     & 91.0 &  0.3 &  4.7 & 298  \\
\bottomrule
\end{tabular}
\end{table}

The following are examples of \WSM{} failures, alongside successful solutions produced by \QAM{} on the same problems.

\paragraph{Example 1 (\texttt{doc\_id} 679; gold 576).}
\emph{If 6 potatoes makes 36 hash browns, how many hash browns can you make out
of 96 potatoes?}

WSM repeats the premise verbatim until truncation and never performs the
division;\:
\begin{small}
\begin{verbatim}
6 potatoes makes 36 hash browns
6 potatoes makes 36 hash browns
...  (28 identical lines, truncated at the 256-token limit)
\end{verbatim}
\end{small}
QAM:
\begin{small}
\begin{verbatim}
6 potatoes makes 36 hash browns
So 1 potato makes 36/6=<<36/6=6>>6
96 potatoes makes 96*6=<<96*6=576>>576
#### 576
\end{verbatim}
\end{small}

\paragraph{Example 2 (\texttt{doc\_id} 164; gold 15).}
\emph{Jackie is trying to decide whether to do her taxes herself or hire an
accountant. If she does the taxes herself, she'll be able to do 3 fewer hours of
freelance work, losing \$35/hour in missed income. The accountant charges \$90.
How much more money will she have if she hires the accountant?}

In this example, WSM computes the first multiplication correctly, but then
repeats the intermediate step until truncation instead of
subtracting the accountant's fee.
\begin{small}
\begin{verbatim}
If she does the taxes herself, she'll be able to do 3 fewer hours of
freelance work, losing $35/hour in missed income.
If she does the taxes herself, she'll be able to do 3*35=<<3*35=105>>105
hours of freelance work.
...  (the same line repeated 7 times, truncated)
\end{verbatim}
\end{small}
QAM reaches the same intermediate value and completes the derivation:
\begin{small}
\begin{verbatim}
... she'll lose 3*35=<<3*35=105>>105 dollars.
The accountant charges $90. So, she'll have 105-90=<<105-90=15>>15
dollars more if she hires the accountant.
#### 15
\end{verbatim}
\end{small}

\paragraph{Example 3 (\texttt{doc\_id} 1207; gold 64).}
\emph{Deandre caught 3 tunas last Monday, the first tuna he caught weighs 56
kilograms, the second tuna he caught weighs 46 kilograms, and the last tuna he
caught weighs 26 kilograms. If a kilogram of tuna costs \$0.50, how much will he
earn after selling all the tunas to the market?}

WSM echoes the question as a single sentence six and a half times before
truncation, and \texttt{flexible-extract} returns the item count \texttt{3}. QAM
prices each fish ($56\times0.5=28$, $46\times0.5=23$, $26\times0.5=13$) and sums
them to \texttt{\#\#\#\# 64}.

\subsection{SmolLM3 Last-15}
\label{app:smol_output_diagnostics}

Table~\ref{tab:smol_output_diagnostics} summarizes the
interpolation endpoint diagnostics. Unlike Prelude, all three
SmolLM3 WSM endpoints have low no-newline rates and median
generation lengths close to those of QAM.

\begin{table}[htbp]
\centering
\small
\caption{SmolLM3 Last-15 diagnostics on the 1,319-question
GSM8K test set.}
\label{tab:smol_output_diagnostics}
\setlength{\tabcolsep}{4pt}
\begin{tabular}{@{}llrrrrr@{}}
\toprule
Profile & Method & Flexible & Strict & No newline & Loop & Median length \\
\midrule
Linear & WSM & 33.4 & 30.6 & 0.8 & 9.9 & 213 \\
Linear & QAM & 43.4 & 43.2 & 0.4 & 3.6 & 221 \\
Cosine & WSM & 37.1 & 35.3 & 0.2 & 5.9 & 233 \\
Cosine & QAM & 45.4 & 44.8 & 0.4 & 2.9 & 252 \\
$1-\sqrt{\cdot}$ & WSM & 32.8 & 29.9 & 1.4 & 11.7 & 218 \\
$1-\sqrt{\cdot}$ & QAM & 41.5 & 41.0 & 0.4 & 3.6 & 231 \\
\bottomrule
\end{tabular}
\end{table}

\clearpage

\section{Maximum-Entropy Diagnostics}
\label{app:matched_second_moment_controls}

These diagnostics compare coefficient vectors at specified index
moments on the same saved checkpoint grid. Their purpose is to test
whether those moments determine downstream scores.
\subsection{Maximum-entropy controls along the interpolation}
\label{app:offpath_controls}

Use zero-based indices $i=0,\ldots,n-1$, with $n=15$ for SmolLM3
and $n=40$ for Prelude. For each of the three profiles and
$\lambda\in\{0.25,0.50,0.75\}$, form
$p^{(\lambda)}=(1-\lambda)c+\lambda q$ and set
\[
\mu=\sum_i i p_i^{(\lambda)},\qquad
v=(1-\lambda)\operatorname{Var}_c(I)+\lambda\operatorname{Var}_q(I).
\]
The control is the maximum-entropy distribution on this finite grid
subject to $\mathbb E_r I=\mu$ and $\operatorname{Var}_r(I)=v$:
\[
r_i=\frac{\exp(a i+b i^2)}{\sum_{j=0}^{n-1}\exp(a j+b j^2)}.
\]
The parameters enforce the two constraints; equivalently, they
minimize the convex function
\[
\log\!\sum_{i=0}^{n-1}\exp(a i+b i^2)-a\mu-b(\mu^2+v).
\]
The gradient gives the residuals in the first and second raw moments,
so this specifies the distribution without task-score fitting.
The target variance is that of the intermediate mixture, generally
larger than the QAM variance at $\lambda=1$. On Prelude, the target
means in one-based indexing are $20.5000$, $25.9615$, and $13.8652$
for linear, cosine, and $1-\sqrt{\cdot}$; subtract one for the
zero-based indices above. Merges use float32 accumulation and bfloat16 export.

Table~\ref{tab:offpath_controls} reports all paired observations.
On SmolLM3, MaxEnt wins six of nine flexible-extraction pairs; under
strict extraction it wins four, loses four, and ties one. On Prelude,
the mixture wins seven of nine flexible and six of nine strict pairs.
The target moments therefore do not specify a unique task score or
ranking. This comparison is a diagnostic, not a causal isolation of
the theoretical truncation error.

\subsection{Separate MaxEnt controls at the QAM endpoint}
\label{app:endpoint-maxent}

For the linear profile on SmolLM3 Last-15 and Prelude Last-40,
separate MaxEnt merges match QAM's index moments.
With zero-based indices, the target pairs $(\mu,v)$ are respectively
\[
\left(7,\frac{112}{45}\right)
\quad\text{and}\quad
\left(19.5,\frac{533}{80}\right).
\]
Each finite-grid MaxEnt distribution uses the same constrained family
defined above with the corresponding target moments, and therefore
satisfies the two-moment conditions in
Theorem~\ref{thm:common-consistency}.

Table~\ref{tab:endpoint-maxent} reports their five-shot GSM8K scores
alongside QAM. The plotted interpolation endpoint at $\lambda=1$ is
QAM itself. The MaxEnt models are separate controls, not additional
points on that interpolation segment.
The comparison only covers this linear profile and benchmark GSM8K.

\begin{table}[H]
\centering\small
\caption{Five-shot GSM8K scores for linear QAM and moment-matched MaxEnt controls.}
\label{tab:endpoint-maxent}
\begin{tabular}{@{}lrrrr@{}}
\toprule
 & \multicolumn{2}{c}{SmolLM3 Last-15} & \multicolumn{2}{c}{Prelude Last-40} \\
\cmidrule(lr){2-3}\cmidrule(l){4-5}
Kernel & Flex & Strict & Flex & Strict \\
\midrule
QAM & 43.37 & 43.21 & 45.41 & 44.66 \\
MaxEnt & 43.75 & 43.52 & 46.10 & 45.34 \\
\bottomrule
\end{tabular}
\end{table}
\section{Additional Performance Comparisons}
\label{app:additional_comparisons}

\subsection{Comparison with Mean Constituent Scores}
\label{app:single_checkpoint_results}
\label{sec:single_checkpoint_envelope}
 
As a complementary comparison, we evaluate \QAM{} against the
average score of the individual checkpoints in each window. As shown in Tables~\ref{tab:stage2_group_vs_envelope} and \ref{tab:prelude_group_vs_envelope}, across two trajectories, three windows, three profiles, and
six capability groups, QAM exceeds this reference in
105 of 108 comparisons.
Single-checkpoint averages are reported to two decimal places.

\begin{table}[H]
\centering
\scriptsize
\caption{SmolLM3 \QAM{} compared with the average single checkpoint.}
\label{tab:stage2_group_vs_envelope}
\setlength{\tabcolsep}{2pt}
\resizebox{\textwidth}{!}{%
\begin{tabular}{lllcccccc}
\toprule
Model & Window & Profile
& Knowledge & Commonsense & Reading & Math & Semantics & Dialogue \\
\midrule
\multirow{12}{*}{SmolLM3}
& \multirow{4}{*}{Last-5}
& Single-checkpoint average
& 58.27 & 58.64 & 76.71 & 34.33 & 63.48 & 69.68 \\
& & Linear
& 60.90 {\tiny(+2.63)}
& 59.86 {\tiny(+1.22)}
& 78.75 {\tiny(+2.04)}
& 41.78 {\tiny(+7.45)}
& 68.09 {\tiny(+4.61)}
& 70.34 {\tiny(+0.66)} \\
& & Cosine
& 60.98 {\tiny(+2.71)}
& 59.85 {\tiny(+1.21)}
& 78.07 {\tiny(+1.36)}
& 41.21 {\tiny(+6.88)}
& 66.30 {\tiny(+2.82)}
& 70.79 {\tiny(+1.11)} \\
& & $1-\sqrt{\cdot}$
& 60.92 {\tiny(+2.65)}
& 59.85 {\tiny(+1.21)}
& 79.27 {\tiny(+2.56)}
& 40.30 {\tiny(+5.97)}
& 68.77 {\tiny(+5.29)}
& 69.90 {\tiny(+0.22)} \\
\cmidrule(lr){2-9}
& \multirow{4}{*}{Last-10}
& Single-checkpoint average
& 58.34 & 58.59 & 76.20 & 34.09 & 62.98 & 69.68 \\
& & Linear
& 61.24 {\tiny(+2.90)}
& 59.69 {\tiny(+1.10)}
& 79.63 {\tiny(+3.43)}
& 43.71 {\tiny(+9.62)}
& 68.40 {\tiny(+5.42)}
& 69.56 {\tiny(-0.12)} \\
& & Cosine
& 61.10 {\tiny(+2.76)}
& 59.78 {\tiny(+1.19)}
& 79.42 {\tiny(+3.22)}
& 41.17 {\tiny(+7.08)}
& 69.14 {\tiny(+6.16)}
& 69.81 {\tiny(+0.13)} \\
& & $1-\sqrt{\cdot}$
& 61.16 {\tiny(+2.82)}
& 59.77 {\tiny(+1.18)}
& 78.65 {\tiny(+2.45)}
& 43.78 {\tiny(+9.69)}
& 67.91 {\tiny(+4.93)}
& 69.96 {\tiny(+0.28)} \\
\cmidrule(lr){2-9}
& \multirow{4}{*}{Last-15}
& Single-checkpoint average
& 58.11 & 58.53 & 76.63 & 33.37 & 62.23 & 69.66 \\
& & Linear
& 61.08 {\tiny(+2.97)}
& 59.53 {\tiny(+1.00)}
& 76.51 {\tiny(-0.12)}
& 43.29 {\tiny(+9.92)}
& 68.46 {\tiny(+6.23)}
& 69.63 {\tiny(-0.03)} \\
& & Cosine
& 61.41 {\tiny(+3.30)}
& 59.83 {\tiny(+1.30)}
& 79.97 {\tiny(+3.34)}
& 45.11 {\tiny(+11.74)}
& 68.77 {\tiny(+6.54)}
& 69.72 {\tiny(+0.06)} \\
& & $1-\sqrt{\cdot}$
& 60.97 {\tiny(+2.86)}
& 59.51 {\tiny(+0.98)}
& 77.65 {\tiny(+1.02)}
& 41.24 {\tiny(+7.87)}
& 67.16 {\tiny(+4.93)}
& 69.95 {\tiny(+0.29)} \\
\bottomrule
\end{tabular}
}
\end{table}
\begin{table}[H]
\centering
\scriptsize
\caption{OpenEuroLLM Prelude \QAM{} compared with the average single checkpoint.}
\label{tab:prelude_group_vs_envelope}
\setlength{\tabcolsep}{1.5pt}
\resizebox{\textwidth}{!}{%
\begin{tabular}{lllcccccc}
\toprule
Model & Window & Profile
& Knowledge & Commonsense & Reading & Math & Semantics & Dialogue \\
\midrule
\multirow{12}{*}{OpenEuroLLM}
& \multirow{4}{*}{Last-10}
& Single-checkpoint average
& 61.16 & 61.94 & 79.52 & 33.16 & 57.56 & 71.65 \\
& & Linear
& 64.55 {\tiny(+3.39)}
& 63.44 {\tiny(+1.50)}
& 81.87 {\tiny(+2.35)}
& 42.80 {\tiny(+9.64)}
& 64.19 {\tiny(+6.63)}
& 72.75 {\tiny(+1.10)} \\
& & Cosine
& 63.86 {\tiny(+2.70)}
& 63.22 {\tiny(+1.28)}
& 81.41 {\tiny(+1.89)}
& 41.43 {\tiny(+8.27)}
& 61.53 {\tiny(+3.97)}
& 72.04 {\tiny(+0.39)} \\
& & $1-\sqrt{\cdot}$
& 64.19 {\tiny(+3.03)}
& 63.42 {\tiny(+1.48)}
& 82.32 {\tiny(+2.80)}
& 42.65 {\tiny(+9.49)}
& 60.22 {\tiny(+2.66)}
& 72.42 {\tiny(+0.77)} \\
\cmidrule(lr){2-9}
& \multirow{4}{*}{Last-20}
& Single-checkpoint average
& 61.09 & 61.94 & 79.66 & 33.37 & 58.65 & 71.42 \\
& & Linear
& 64.06 {\tiny(+2.97)}
& 63.77 {\tiny(+1.83)}
& 82.97 {\tiny(+3.31)}
& 44.47 {\tiny(+11.10)}
& 65.74 {\tiny(+7.09)}
& 72.70 {\tiny(+1.28)} \\
& & Cosine
& 64.46 {\tiny(+3.37)}
& 63.53 {\tiny(+1.59)}
& 82.66 {\tiny(+3.00)}
& 44.13 {\tiny(+10.76)}
& 63.81 {\tiny(+5.16)}
& 72.50 {\tiny(+1.08)} \\
& & $1-\sqrt{\cdot}$
& 64.24 {\tiny(+3.15)}
& 63.65 {\tiny(+1.71)}
& 82.23 {\tiny(+2.57)}
& 42.57 {\tiny(+9.20)}
& 67.97 {\tiny(+9.32)}
& 72.22 {\tiny(+0.80)} \\
\cmidrule(lr){2-9}
& \multirow{4}{*}{Last-40}
& Single-checkpoint average
& 61.05 & 62.00 & 79.47 & 33.57 & 59.34 & 71.60 \\
& & Linear
& 64.27 {\tiny(+3.22)}
& 63.41 {\tiny(+1.41)}
& 82.20 {\tiny(+2.73)}
& 45.03 {\tiny(+11.46)}
& 66.73 {\tiny(+7.39)}
& 72.85 {\tiny(+1.25)} \\
& & Cosine
& 64.39 {\tiny(+3.34)}
& 63.64 {\tiny(+1.64)}
& 82.29 {\tiny(+2.82)}
& 44.47 {\tiny(+10.90)}
& 67.91 {\tiny(+8.57)}
& 72.25 {\tiny(+0.65)} \\
& & $1-\sqrt{\cdot}$
& 64.03 {\tiny(+2.98)}
& 63.42 {\tiny(+1.42)}
& 82.78 {\tiny(+3.31)}
& 42.57 {\tiny(+9.00)}
& 65.12 {\tiny(+5.78)}
& 73.31 {\tiny(+1.71)} \\
\bottomrule
\end{tabular}
}
\end{table}
Notice that the improvements differ by capability in both their size
and consistency across configurations. Knowledge and Commonsense show consistent gains across all
tested windows and profiles, ranging from 2.63 to 3.39
percentage points and from 0.98 to 1.83, respectively.
Semantics also improves in every configuration, although
the magnitude varies more widely, from 2.66 to 9.32 points.

Math shows particularly large gains on both trajectories.
Improvements range from 5.97 to 11.74 percentage points on
SmolLM3 and from 8.27 to 11.46 on Prelude, with gains
present in every tested configuration.

Reading and Dialogue improve in 17 and 16 of the 18
configurations, respectively, with only small observed
regressions.
All three negative differences occur on SmolLM3,
comprising a 0.12-point decrease in Reading and decreases
of 0.12 and 0.03 MRR score points in Dialogue.

This comparison establishes improvements over the mean constituent
score. Appendix~\ref{app:checkpoint-direct} reports direct comparisons with
the latest and task-best constituent checkpoints.

\subsection{Long-window GSM8K scores}
\label{app:long_window_gsm_scores}
The large Math gap on Prelude Last-40 reflects a sharp
decline in \WSM{}'s GSM8K performance.
Table~\ref{tab:prelude-n40-gsm-raw} reports the paired raw scores.
WSM scores only $11.83$--$14.03$ under flexible extraction and
$5.38$--$9.48$ under strict extraction, whereas QAM scores
$42.84$--$45.41$ and $42.30$--$44.66$, respectively.
The individual checkpoints in this window score
$29.4$--$37.7$ under flexible extraction.
Across all three profiles, \WSM{} falls below every constituent checkpoint, while \QAM{} exceeds all of them under this metric.

\begin{table}[H]
\centering
\small
\caption{GSM8K scores for Prelude Last-40 under flexible
and strict answer extraction.}
\label{tab:prelude-n40-gsm-raw}
\vspace{2pt}
\setlength{\tabcolsep}{10pt}
\begin{tabular}{lcccccc}
\toprule
& \multicolumn{2}{c}{Linear}
& \multicolumn{2}{c}{Cosine}
& \multicolumn{2}{c}{$1-\sqrt{\cdot}$} \\
\cmidrule(lr){2-3}
\cmidrule(lr){4-5}
\cmidrule(lr){6-7}
Metric
& \WSM{} & \QAM{}
& \WSM{} & \QAM{}
& \WSM{} & \QAM{} \\
\midrule
GSM8K flexible
& 11.83 & \textbf{45.41}
& 14.03 & \textbf{44.66}
& 12.66 & \textbf{42.84} \\
GSM8K strict
& 5.38 & \textbf{44.66}
& 9.48 & \textbf{44.28}
& 7.13 & \textbf{42.30} \\
\addlinespace[0.2em]
flexible $-$ strict
& 6.44 & 0.76
& 4.55 & 0.38
& 5.53 & 0.53 \\
\bottomrule
\end{tabular}
\end{table}

\subsection{Task-wise best-observed comparisons}
\label{app:best_observed_tasks}

Table~\ref{tab:best_observed_tasks} makes the cross-window comparison
explicit. Each task is considered separately: WSM may use any of
the three windows and three profiles, while QAM may use either the
intermediate or long window and any of the three profiles.
Different rows can therefore select different merged models.
These are best observed scores over the evaluated grid, and different
tasks may select different configurations.

\begin{table}[H]
\centering
\small
\caption{Task-wise best observed scores. WSM uses all tested windows;
QAM uses intermediate and long windows. Scores are percentages, except MuTual
which is $100\times\mathrm{MRR}$.}
\label{tab:best_observed_tasks}
\setlength{\tabcolsep}{3pt}
\begin{tabular}{@{}lrrr rrr@{}}
\toprule
& \multicolumn{3}{c}{SmolLM3} & \multicolumn{3}{c}{Prelude}\\
\cmidrule(lr){2-4}\cmidrule(l){5-7}
Task / metric & WSM & QAM & $\Delta$ & WSM & QAM & $\Delta$\\
\midrule
ARC-Challenge & 47.61 & 48.12 & +0.51 & 50.94 & 50.68 & -0.26 \\
ARC-Easy & 79.46 & 79.84 & +0.38 & 80.60 & 80.85 & +0.25 \\
BoolQ & 79.17 & 79.97 & +0.80 & 83.06 & 82.97 & -0.09 \\
COPA & 87.00 & 88.00 & +1.00 & 92.00 & 92.00 & 0.00 \\
HellaSwag & 54.89 & 55.20 & +0.31 & 58.97 & 59.28 & +0.31 \\
OpenBookQA & 35.80 & 34.80 & -1.00 & 35.80 & 35.60 & -0.20 \\
PIQA & 78.62 & 78.35 & -0.27 & 79.76 & 80.58 & +0.82 \\
SciQ & 95.60 & 95.40 & -0.20 & 95.90 & 96.00 & +0.10 \\
SST-2 & 85.78 & 85.89 & +0.11 & 84.52 & 83.94 & -0.58 \\
WiC & 51.25 & 50.16 & -1.09 & 50.31 & 50.78 & +0.47 \\
WinoGrande & 69.93 & 69.06 & -0.87 & 73.16 & 74.27 & +1.11 \\
WSC & 47.12 & 47.12 & 0.00 & 36.54 & 42.31 & +5.77 \\
MMLU (5-shot) & 57.80 & 57.99 & +0.19 & 62.31 & 61.88 & -0.43 \\
GSM8K flexible & 42.23 & 45.41 & +3.18 & 43.06 & 45.41 & +2.35 \\
GSM8K strict & 41.24 & 44.81 & +3.57 & 42.76 & 44.66 & +1.90 \\
MuTual & 70.65 & 69.96 & -0.69
       & 72.79 & 73.31 & +0.52\\
\bottomrule
\end{tabular}
\end{table}

GSM8K counts as one task as flexible and strict extraction give the
same win/tie/loss classification. Including MuTual, the totals are
8 wins, 1 tie, and 6 losses for SmolLM3, and 9 wins, 1 tie, and
5 losses for Prelude. Several positive differences are small;
these counts describe direction without establishing statistical
significance or a universal advantage.

For SmolLM3, the best WSM MuTual score is $70.65$ at Last-5
cosine, while the best intermediate/long QAM score is $69.96$
at Last-10 $1-\sqrt{\cdot}$. For Prelude, the best WSM score is
$72.79$ at Last-10 cosine, while the best intermediate/long QAM
score is $73.31$ at Last-40 $1-\sqrt{\cdot}$.

\clearpage

\section{Comparison with Individual Checkpoints}
\label{app:checkpoint-comparisons}
\label{app:checkpoint-direct}

We compare all 18 QAM configurations with individual checkpoints
across the 15 evaluation tasks. Each merge is compared with every
checkpoint in its own window, the latest checkpoint, and the
highest-scoring checkpoint for each task within that window.
The latest checkpoint is shared across windows within each trajectory.
The task-wise best checkpoint may differ across tasks.

Scores use a 0--100 scale, including MuTual MRR multiplied by 100.
GSM8K averages flexible and strict extraction and counts as one task.
Wins, ties, and losses are determined after rounding both scores
to two decimal places.

Table~\ref{tab:checkpoint-grid} summarizes all three windows and
all three profiles for each trajectory. Across this complete grid,
QAM wins 1092 of 1350 task--checkpoint comparisons on SmolLM3
(80.9\%) and 2716 of 3150 on Prelude (86.2\%).
Across the 270 task--configuration comparisons, QAM beats the latest
checkpoint in 219, ties in 14, and loses in 37.
Against the task-wise best checkpoint within each window,
it wins 147, ties 5, and loses 118.
Tables~\ref{tab:checkpoint-smollm3} and~\ref{tab:checkpoint-prelude}
report the task scores and constituent-checkpoint comparisons
for all 18 configurations.

\begin{table}[htbp]
\centering
\small
\caption{QAM versus individual checkpoints, broken down by window and profile.}
\label{tab:checkpoint-grid}
\setlength{\tabcolsep}{4pt}
\begin{tabular}{@{}lllrrr@{}}
\toprule
Trajectory & Window & Profile & All & Latest & Best \\
\midrule
SmolLM3 & Last-5 & Linear & 68/0/7 & 14/0/1 & 12/0/3 \\
SmolLM3 & Last-5 & Cosine & 66/2/7 & 14/0/1 & 11/1/3 \\
SmolLM3 & Last-5 & $1-\sqrt{\cdot}$ & 58/5/12 & 10/2/3 & 8/1/6 \\
SmolLM3 & Last-10 & Linear & 116/5/29 & 10/1/4 & 7/0/8 \\
SmolLM3 & Last-10 & Cosine & 123/4/23 & 12/0/3 & 7/0/8 \\
SmolLM3 & Last-10 & $1-\sqrt{\cdot}$ & 116/7/27 & 9/1/5 & 6/1/8 \\
SmolLM3 & Last-15 & Linear & 174/9/42 & 10/1/4 & 6/0/9 \\
SmolLM3 & Last-15 & Cosine & 187/7/31 & 12/0/3 & 7/0/8 \\
SmolLM3 & Last-15 & $1-\sqrt{\cdot}$ & 184/9/32 & 11/0/4 & 6/0/9 \\
\midrule
\multicolumn{3}{l}{SmolLM3 total}
& 1092/48/210 & 102/5/28 & 70/3/62 \\
\midrule
Prelude & Last-10 & Linear & 137/3/10 & 14/1/0 & 10/0/5 \\
Prelude & Last-10 & Cosine & 128/5/17 & 12/2/1 & 8/0/7 \\
Prelude & Last-10 & $1-\sqrt{\cdot}$ & 134/2/14 & 13/0/2 & 9/1/5 \\
Prelude & Last-20 & Linear & 251/19/30 & 13/1/1 & 10/0/5 \\
Prelude & Last-20 & Cosine & 257/3/40 & 13/1/1 & 9/0/6 \\
Prelude & Last-20 & $1-\sqrt{\cdot}$ & 261/11/28 & 13/1/1 & 8/0/7 \\
Prelude & Last-40 & Linear & 500/41/59 & 12/2/1 & 7/1/7 \\
Prelude & Last-40 & Cosine & 548/16/36 & 13/1/1 & 9/0/6 \\
Prelude & Last-40 & $1-\sqrt{\cdot}$ & 500/21/79 & 14/0/1 & 7/0/8 \\
\midrule
\multicolumn{3}{l}{Prelude total}
& 2716/121/313 & 117/9/9 & 77/2/56 \\
\midrule
\multicolumn{3}{l}{Combined total}
& 3808/169/523 & 219/14/37 & 147/5/118 \\
\bottomrule
\end{tabular}
\end{table}

\begin{table}[p]
\centering
\small
\setlength{\tabcolsep}{3pt}
\caption{SmolLM3 QAM versus individual checkpoints
for all windows and profiles.}
\label{tab:checkpoint-smollm3}
\begin{tabular}{@{}lrr*{3}{rr}@{}}
\toprule
& & & \multicolumn{2}{c}{Linear}
& \multicolumn{2}{c}{Cosine}
& \multicolumn{2}{c}{$1-\sqrt{\cdot}$} \\
\cmidrule(lr){4-5}\cmidrule(lr){6-7}\cmidrule(l){8-9}
Task & Latest & Best & QAM & W/T/L & QAM & W/T/L & QAM & W/T/L \\
\midrule
\multicolumn{9}{c}{\textbf{Last-5}} \\
\midrule
ARC-Challenge & 45.39 & 46.59 & 47.01 & 5/0/0 & 47.10 & 5/0/0 & 47.35 & 5/0/0 \\
ARC-Easy & 77.69 & 77.78 & 79.00 & 5/0/0 & 79.50 & 5/0/0 & 79.21 & 5/0/0 \\
OpenBookQA & 33.60 & 34.40 & 36.00 & 5/0/0 & 36.40 & 5/0/0 & 34.40 & 4/1/0 \\
SciQ & 95.00 & 95.10 & 95.40 & 5/0/0 & 95.50 & 5/0/0 & 95.00 & 3/1/1 \\
MMLU & 54.69 & 54.98 & 57.42 & 5/0/0 & 57.42 & 5/0/0 & 57.48 & 5/0/0 \\
HellaSwag & 53.84 & 53.93 & 55.07 & 5/0/0 & 55.06 & 5/0/0 & 55.13 & 5/0/0 \\
PIQA & 77.80 & 77.91 & 78.02 & 5/0/0 & 78.07 & 5/0/0 & 78.24 & 5/0/0 \\
COPA & 86.00 & 89.00 & 87.00 & 3/0/2 & 88.00 & 3/1/1 & 88.00 & 3/1/1 \\
WinoGrande & 68.35 & 69.14 & 69.30 & 5/0/0 & 69.14 & 4/1/0 & 68.35 & 2/1/2 \\
BoolQ & 77.95 & 77.95 & 78.75 & 5/0/0 & 78.07 & 5/0/0 & 79.27 & 5/0/0 \\
GSM8K & 31.05 & 36.51 & 41.78 & 5/0/0 & 41.21 & 5/0/0 & 40.30 & 5/0/0 \\
SST-2 & 65.60 & 84.63 & 83.37 & 4/0/1 & 80.05 & 3/0/2 & 84.86 & 5/0/0 \\
WiC & 50.78 & 50.78 & 51.10 & 5/0/0 & 51.25 & 5/0/0 & 49.84 & 1/1/3 \\
WSC & 67.31 & 67.31 & 44.23 & 1/0/4 & 43.27 & 1/0/4 & 50.00 & 2/0/3 \\
MuTual & 70.33 & 70.33 & 70.34 & 5/0/0 & 70.79 & 5/0/0 & 69.90 & 3/0/2 \\
\midrule
\multicolumn{9}{c}{\textbf{Last-10}} \\
\midrule
ARC-Challenge & 45.39 & 46.59 & 48.04 & 10/0/0 & 47.44 & 10/0/0 & 47.61 & 10/0/0 \\
ARC-Easy & 77.69 & 78.24 & 79.84 & 10/0/0 & 79.00 & 10/0/0 & 79.71 & 10/0/0 \\
OpenBookQA & 33.60 & 35.80 & 34.60 & 9/0/1 & 34.80 & 9/0/1 & 34.20 & 7/1/2 \\
SciQ & 95.00 & 95.40 & 95.20 & 8/1/1 & 95.10 & 7/1/2 & 95.40 & 9/1/0 \\
MMLU & 54.69 & 55.96 & 57.72 & 10/0/0 & 57.72 & 10/0/0 & 57.67 & 10/0/0 \\
HellaSwag & 53.84 & 53.94 & 54.98 & 10/0/0 & 55.05 & 10/0/0 & 55.20 & 10/0/0 \\
PIQA & 77.80 & 77.91 & 78.24 & 10/0/0 & 78.02 & 10/0/0 & 77.69 & 8/0/2 \\
COPA & 86.00 & 89.00 & 86.00 & 2/3/5 & 88.00 & 6/2/2 & 86.00 & 2/3/5 \\
WinoGrande & 68.35 & 69.14 & 68.03 & 5/0/5 & 68.59 & 8/0/2 & 67.96 & 5/0/5 \\
BoolQ & 77.95 & 77.95 & 79.63 & 10/0/0 & 79.42 & 10/0/0 & 78.65 & 10/0/0 \\
GSM8K & 31.05 & 36.51 & 43.71 & 10/0/0 & 41.17 & 10/0/0 & 43.78 & 10/0/0 \\
SST-2 & 65.60 & 87.16 & 84.75 & 9/0/1 & 85.89 & 9/0/1 & 83.37 & 8/0/2 \\
WiC & 50.78 & 50.78 & 49.84 & 4/1/5 & 49.84 & 4/1/5 & 50.16 & 7/2/1 \\
WSC & 67.31 & 67.31 & 45.19 & 4/0/6 & 47.12 & 4/0/6 & 47.12 & 4/0/6 \\
MuTual & 70.33 & 70.33 & 69.56 & 5/0/5 & 69.81 & 6/0/4 & 69.96 & 6/0/4 \\
\midrule
\multicolumn{9}{c}{\textbf{Last-15}} \\
\midrule
ARC-Challenge & 45.39 & 46.59 & 47.18 & 15/0/0 & 47.87 & 15/0/0 & 48.12 & 15/0/0 \\
ARC-Easy & 77.69 & 78.24 & 79.84 & 15/0/0 & 79.80 & 15/0/0 & 79.71 & 15/0/0 \\
OpenBookQA & 33.60 & 35.80 & 34.60 & 13/0/2 & 34.60 & 13/0/2 & 34.80 & 13/0/2 \\
SciQ & 95.00 & 95.60 & 95.10 & 8/2/5 & 95.10 & 8/2/5 & 95.20 & 10/2/3 \\
MMLU & 54.69 & 55.96 & 57.58 & 15/0/0 & 57.99 & 15/0/0 & 57.36 & 15/0/0 \\
HellaSwag & 53.84 & 53.94 & 54.76 & 15/0/0 & 55.09 & 15/0/0 & 54.58 & 15/0/0 \\
PIQA & 77.80 & 77.91 & 77.97 & 15/0/0 & 77.97 & 15/0/0 & 78.35 & 15/0/0 \\
COPA & 86.00 & 89.00 & 86.00 & 5/3/7 & 87.00 & 8/3/4 & 87.00 & 8/3/4 \\
WinoGrande & 68.35 & 69.14 & 68.51 & 12/1/2 & 68.98 & 14/0/1 & 69.06 & 14/0/1 \\
BoolQ & 77.95 & 79.20 & 76.51 & 6/0/9 & 79.97 & 15/0/0 & 77.65 & 13/0/2 \\
GSM8K & 31.05 & 36.51 & 43.29 & 15/0/0 & 45.11 & 15/0/0 & 41.24 & 15/0/0 \\
SST-2 & 65.60 & 87.16 & 84.98 & 14/0/1 & 85.09 & 14/0/1 & 83.14 & 13/0/2 \\
WiC & 50.78 & 50.78 & 50.16 & 11/3/1 & 50.00 & 9/2/4 & 50.16 & 11/3/1 \\
WSC & 67.31 & 67.31 & 42.31 & 6/0/9 & 47.12 & 7/0/8 & 37.50 & 2/1/12 \\
MuTual & 70.33 & 70.33 & 69.63 & 9/0/6 & 69.72 & 9/0/6 & 69.95 & 10/0/5 \\
\bottomrule
\end{tabular}
\end{table}

\begin{table}[p]
\centering
\small
\setlength{\tabcolsep}{3pt}
\caption{Prelude QAM versus individual checkpoints
for all windows and profiles.}
\label{tab:checkpoint-prelude}
\begin{tabular}{@{}lrr*{3}{rr}@{}}
\toprule
& & & \multicolumn{2}{c}{Linear}
& \multicolumn{2}{c}{Cosine}
& \multicolumn{2}{c}{$1-\sqrt{\cdot}$} \\
\cmidrule(lr){4-5}\cmidrule(lr){6-7}\cmidrule(l){8-9}
Task & Latest & Best & QAM & W/T/L & QAM & W/T/L & QAM & W/T/L \\
\midrule
\multicolumn{9}{c}{\textbf{Last-10}} \\
\midrule
ARC-Challenge & 47.61 & 48.46 & 50.77 & 10/0/0 & 49.23 & 10/0/0 & 49.91 & 10/0/0 \\
ARC-Easy & 79.08 & 79.50 & 80.68 & 10/0/0 & 80.30 & 10/0/0 & 80.26 & 10/0/0 \\
OpenBookQA & 31.80 & 33.20 & 33.80 & 10/0/0 & 35.00 & 10/0/0 & 34.80 & 10/0/0 \\
SciQ & 95.50 & 95.90 & 95.60 & 7/0/3 & 95.50 & 6/1/3 & 95.40 & 5/1/4 \\
MMLU & 58.82 & 58.82 & 61.86 & 10/0/0 & 61.07 & 10/0/0 & 61.49 & 10/0/0 \\
HellaSwag & 57.33 & 58.03 & 59.16 & 10/0/0 & 58.92 & 10/0/0 & 59.13 & 10/0/0 \\
PIQA & 78.51 & 79.38 & 79.60 & 10/0/0 & 79.43 & 10/0/0 & 79.60 & 10/0/0 \\
COPA & 90.00 & 93.00 & 90.00 & 5/3/2 & 90.00 & 5/3/2 & 91.00 & 8/0/2 \\
WinoGrande & 68.35 & 71.74 & 71.82 & 10/0/0 & 71.67 & 9/0/1 & 71.74 & 9/1/0 \\
BoolQ & 79.05 & 80.86 & 81.87 & 10/0/0 & 81.41 & 10/0/0 & 82.32 & 10/0/0 \\
GSM8K & 33.85 & 35.22 & 42.80 & 10/0/0 & 41.43 & 10/0/0 & 42.65 & 10/0/0 \\
SST-2 & 62.50 & 79.70 & 74.66 & 9/0/1 & 72.02 & 9/0/1 & 69.27 & 9/0/1 \\
WiC & 51.72 & 56.58 & 52.66 & 8/0/2 & 50.47 & 3/1/6 & 51.25 & 5/0/5 \\
WSC & 36.54 & 59.62 & 47.12 & 8/0/2 & 41.35 & 8/0/2 & 39.42 & 8/0/2 \\
MuTual & 71.80 & 72.24 & 72.75 & 10/0/0 & 72.04 & 8/0/2 & 72.42 & 10/0/0 \\
\midrule
\multicolumn{9}{c}{\textbf{Last-20}} \\
\midrule
ARC-Challenge & 47.61 & 48.72 & 50.68 & 20/0/0 & 49.66 & 20/0/0 & 48.55 & 19/0/1 \\
ARC-Easy & 79.08 & 79.50 & 80.43 & 20/0/0 & 80.85 & 20/0/0 & 80.13 & 20/0/0 \\
OpenBookQA & 31.80 & 34.40 & 33.80 & 18/1/1 & 35.60 & 20/0/0 & 34.00 & 19/0/1 \\
SciQ & 95.50 & 96.30 & 95.80 & 13/4/3 & 95.50 & 9/1/10 & 95.80 & 13/4/3 \\
MMLU & 58.82 & 58.82 & 61.23 & 20/0/0 & 61.74 & 20/0/0 & 61.69 & 20/0/0 \\
HellaSwag & 57.33 & 58.03 & 59.28 & 20/0/0 & 59.23 & 20/0/0 & 59.22 & 20/0/0 \\
PIQA & 78.51 & 79.54 & 80.25 & 20/0/0 & 79.60 & 20/0/0 & 80.14 & 20/0/0 \\
COPA & 90.00 & 93.00 & 92.00 & 17/1/2 & 91.00 & 17/0/3 & 90.00 & 10/7/3 \\
WinoGrande & 68.35 & 72.69 & 73.24 & 20/0/0 & 72.14 & 18/0/2 & 72.77 & 20/0/0 \\
BoolQ & 79.05 & 80.86 & 82.97 & 20/0/0 & 82.66 & 20/0/0 & 82.23 & 20/0/0 \\
GSM8K & 33.85 & 35.52 & 44.47 & 20/0/0 & 44.13 & 20/0/0 & 42.57 & 20/0/0 \\
SST-2 & 62.50 & 80.39 & 80.62 & 20/0/0 & 76.83 & 17/0/3 & 83.94 & 20/0/0 \\
WiC & 51.72 & 57.99 & 50.16 & 3/3/14 & 50.31 & 6/0/14 & 50.63 & 7/0/13 \\
WSC & 36.54 & 59.62 & 36.54 & 0/10/10 & 37.50 & 10/2/8 & 40.38 & 14/0/6 \\
MuTual & 71.80 & 72.24 & 72.70 & 20/0/0 & 72.50 & 20/0/0 & 72.22 & 19/0/1 \\
\midrule
\multicolumn{9}{c}{\textbf{Last-40}} \\
\midrule
ARC-Challenge & 47.61 & 49.15 & 49.15 & 39/1/0 & 48.81 & 39/0/1 & 47.95 & 30/1/9 \\
ARC-Easy & 79.08 & 79.63 & 79.59 & 39/0/1 & 79.88 & 40/0/0 & 79.92 & 40/0/0 \\
OpenBookQA & 31.80 & 34.40 & 33.20 & 34/3/3 & 34.80 & 40/0/0 & 33.40 & 37/0/3 \\
SciQ & 95.50 & 96.30 & 96.00 & 37/2/1 & 95.90 & 34/3/3 & 95.60 & 20/5/15 \\
MMLU & 58.82 & 59.00 & 61.79 & 40/0/0 & 61.88 & 40/0/0 & 61.53 & 40/0/0 \\
HellaSwag & 57.33 & 58.20 & 58.88 & 40/0/0 & 59.12 & 40/0/0 & 58.82 & 40/0/0 \\
PIQA & 78.51 & 80.09 & 80.03 & 38/0/2 & 80.58 & 40/0/0 & 79.54 & 34/1/5 \\
COPA & 90.00 & 93.00 & 90.00 & 20/11/9 & 90.00 & 20/11/9 & 91.00 & 31/3/6 \\
WinoGrande & 68.35 & 72.69 & 73.09 & 40/0/0 & 72.85 & 40/0/0 & 74.27 & 40/0/0 \\
BoolQ & 79.05 & 81.25 & 82.20 & 40/0/0 & 82.29 & 40/0/0 & 82.78 & 40/0/0 \\
GSM8K & 33.85 & 37.57 & 45.03 & 40/0/0 & 44.47 & 40/0/0 & 42.57 & 40/0/0 \\
SST-2 & 62.50 & 80.50 & 82.45 & 40/0/0 & 83.49 & 40/0/0 & 79.36 & 36/0/4 \\
WiC & 51.72 & 57.99 & 50.16 & 13/5/22 & 50.78 & 22/2/16 & 50.16 & 13/5/22 \\
WSC & 36.54 & 59.62 & 36.54 & 0/19/21 & 42.31 & 34/0/6 & 37.50 & 19/6/15 \\
MuTual & 71.80 & 72.37 & 72.85 & 40/0/0 & 72.25 & 39/0/1 & 73.31 & 40/0/0 \\
\bottomrule
\end{tabular}
\end{table}
\clearpage
\section{Individual-Checkpoint Results}\label{app:full results}

Figures~\ref{fig:single ckpt score smollm3} and
\ref{fig:single ckpt score prelude} show the individual-checkpoint
benchmark scores on the two trajectories, including MuTual MRR. Window-level mean scores and
the corresponding \QAM{} results are reported in
Tables~\ref{tab:stage2_group_vs_envelope} and
\ref{tab:prelude_group_vs_envelope};
Table~\ref{tab:best_observed_tasks} compares the best observed merge
scores across the main evaluation suite.

\begin{figure}[htbp]
    \centering
    \includegraphics[width=\textwidth,height=0.79\textheight,keepaspectratio]{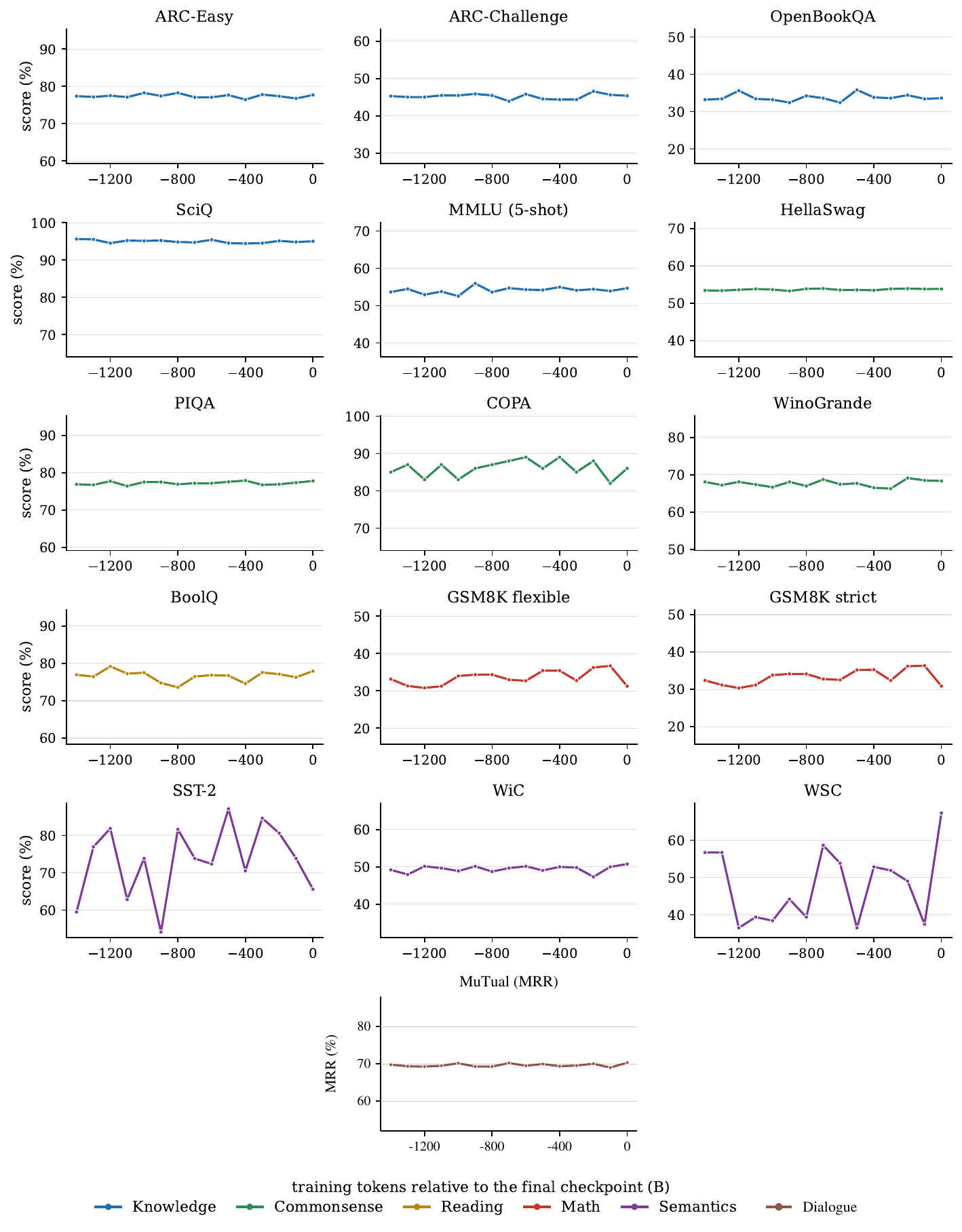}
    \caption{Individual-checkpoint benchmark scores for SmolLM3.}
    \label{fig:single ckpt score smollm3}
\end{figure}

\begin{figure}[htbp]
    \centering
    \includegraphics[width=\textwidth,height=0.79\textheight,keepaspectratio]{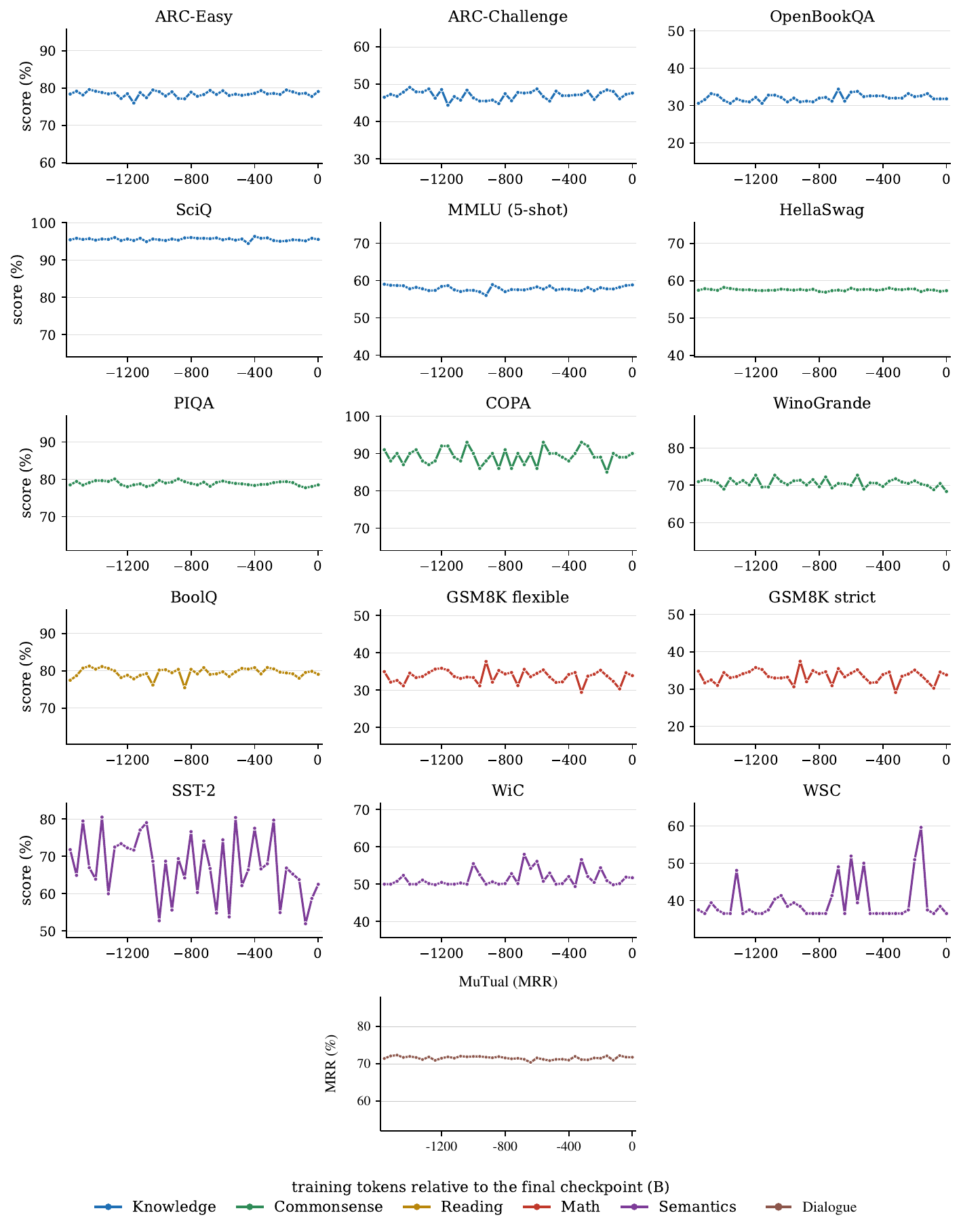}
    \caption{Individual-checkpoint benchmark scores for OpenEuroLLM Prelude.}
    \label{fig:single ckpt score prelude}
\end{figure}

\end{document}